\documentclass{article}

     \PassOptionsToPackage{numbers, compress}{natbib}

\usepackage[preprint]{neurips_2022}

\usepackage[utf8]{inputenc} 
\usepackage[T1]{fontenc}    
\usepackage{hyperref}       
\usepackage{url}            
\usepackage{booktabs}       
\usepackage{amsfonts}       
\usepackage{nicefrac}       
\usepackage{microtype}      
\usepackage{xcolor}         
\usepackage{natbib}
\usepackage{xspace}
\usepackage{mathtools}
\usepackage{subcaption}
\usepackage[normalem]{ulem}

\usepackage{amsmath,amsfonts,bm, amsthm,amssymb}

\def\eqref#1{equation~\ref{#1}}

\def\1{\bm{1}}

\def\va{{\bm{a}}}
\def\vb{{\bm{b}}}

\def\vu{{\bm{u}}}

\def\vw{{\bm{w}}}
\def\vx{{\bm{x}}}

\def\mA{{\bm{A}}}
\def\mB{{\bm{B}}}

\def\mD{{\bm{D}}}

\def\mI{{\bm{I}}}

\def\mP{{\bm{P}}}

\DeclareMathAlphabet{\mathsfit}{\encodingdefault}{\sfdefault}{m}{sl}
\SetMathAlphabet{\mathsfit}{bold}{\encodingdefault}{\sfdefault}{bx}{n}

\def\gD{{\mathcal{D}}}

\def\gN{{\mathcal{N}}}
\def\gO{{\mathcal{O}}}

\def\gU{{\mathcal{U}}}

\newcommand{\E}{\mathbb{E}}

\newcommand{\R}{\mathbb{R}}

\newcommand{\lr}{\alpha}

\newcommand{\Var}{\mathrm{Var}}

\DeclareMathOperator*{\argmin}{arg\,min}

\newtheorem{thm}{Theorem}

\newtheorem{lem}{Lemma}
\newtheorem{corollary}{Corollary}
\newtheorem{assump}{Assumption}

\newtheorem{defn}{Definition}

\newcommand{\efflr}{\widetilde{\lr}}
\newcommand{\fedavg}{\textsc{FedAvg}\xspace}
\newcommand{\measure}{average drift at optimum\xspace}
\newcommand{\pseudograd}{\mathcal{G}}
\newcommand{\effmu}{\widetilde{\mu}}
\newcommand{\effLip}{\widetilde{L}}

\newcommand{\brackets}[1]{\left[#1\right]}
\newcommand{\parenth}[1]{\left(#1\right)}
\newcommand{\vecnorm}[1]{\left\|#1\right\|^2}
\newcommand{\norm}[1]{\left\|#1\right\|}
\newcommand{\inprod}[2]{\left<#1, #2\right>}

\usepackage[colorinlistoftodos,prependcaption,disable]{todonotes}

\newcommand{\GJ}[1]{\todo[color=green!25, inline]{ Gauri: #1} \index{Gauri: !#1}}
\newcommand{\JW}[1]{\todo[color=yellow!25, inline]{ Jianyu: #1} \index{Jianyu: !#1}}
\newcommand{\ZX}[1]{\todo[color=cyan!25, inline]{Zheng: #1} \index{Zheng: !#1}}

\newcommand{\SK}[1]{\todo[color=red!25, inline]{SK: #1} \index{SK: !#1}}

\usepackage{cleveref}
\crefname{equation}{}{}
\Crefname{equation}{}{}
\crefname{thm}{theorem}{theorems}
\Crefname{thm}{Theorem}{Theorems}
\crefname{clm}{claim}{claims}
\Crefname{clm}{Claim}{Claims}
\Crefname{coro}{Corollary}{Corollaries}
\Crefname{lem}{Lemma}{Lemmas}
\Crefname{sec}{Section}{Sections}
\crefname{app}{appendix}{appendices}
\Crefname{app}{Appendix}{Appendices}
\Crefname{part}{Part}{Parts}
\crefname{prop}{proposition}{propositions}
\Crefname{prop}{Proposition}{Propositions}
\Crefname{propty}{Property}{Properties}
\crefname{figure}{fig.}{figures}
\Crefname{figure}{Figure}{Figures}
\crefname{defn}{definition}{definitions}
\Crefname{defn}{Definition}{Definitions}
\crefname{fact}{fact}{facts}
\Crefname{fact}{Fact}{Facts}
\crefname{appendix}{appendix}{appendices}
\Crefname{appendix}{Appendix}{Appendices}
\crefname{algo}{algorithm}{algorithms}
\Crefname{algo}{Algorithm}{Algorithms}
\crefname{algorithm}{algorithm}{algorithms}
\Crefname{algorithm}{Algorithm}{Algorithms}
\crefname{conj}{conjecture}{conjectures}
\Crefname{conj}{Conjecture}{Conjectures}
\crefname{obs}{observation}{observations}
\Crefname{obs}{Observation}{Observations}
\crefname{assump}{assumption}{assumptions}
\Crefname{assump}{Assumption}{Assumptions}
\crefname{rem}{remark}{remarks}
\Crefname{rem}{Remark}{Remarks}

\title{On the Unreasonable Effectiveness of Federated Averaging with Heterogeneous Data}

\author{%
  Jianyu Wang \\
  Carnegie Mellon University\\
   \And
   Rudrajit Das \\
   University of Texas at Austin \\
   \And
   Gauri Joshi \\
   Carnegie Mellon University \\
   \And
   Satyen Kale \\
   Google Research \\
   \And
   Zheng Xu \\
   Google Research \\
   \And
   Tong Zhang \\
   Google Research and HKUST \\
}

\begin{document}

\maketitle

\begin{abstract}
Existing theory predicts that data heterogeneity will degrade the performance of the Federated Averaging (FedAvg) algorithm in federated learning. However, in practice, the simple FedAvg algorithm converges very well. This paper explains the seemingly unreasonable effectiveness of FedAvg that contradicts the previous theoretical predictions. We find that the key assumption of bounded gradient dissimilarity in previous theoretical analyses is too pessimistic to characterize data heterogeneity in practical applications. For a simple quadratic problem, we demonstrate there exist regimes where large gradient dissimilarity does not have any negative impact on the convergence of FedAvg. Motivated by this observation, we propose a new quantity \emph{average drift at optimum} to measure the effects of data heterogeneity, and explicitly use it to present a new theoretical analysis of FedAvg. We show that the average drift at optimum is nearly zero across many real-world federated training tasks, whereas the gradient dissimilarity can be large. And our new analysis suggests FedAvg can have \emph{identical} convergence rates in homogeneous and heterogeneous data settings, and hence, leads to better understanding of its empirical success. 
\end{abstract}

\section{Introduction}
Federated learning (FL) is an emerging distributed training paradigm~\cite{kairouz2019advances,wang2021field}, which enables a large number of clients to collaboratively train a powerful machine learning model without the need of uploading any raw training data. One of the most popular FL algorithms is Federated Averaging (\fedavg), proposed by \cite{mcmahan2016communication}. In each round, a small subset of clients are randomly selected to perform local model training. Then, the local model changes from clients are aggregated at the server to update the global model. This general local-update framework only requires infrequent communication between server and clients, and thus, is especially suitable for FL settings where the communication cost is a major bottleneck. 

Due to its simplicity and empirical effectiveness, \fedavg has become the basis of almost all subsequent FL optimization algorithms. Nonetheless, its convergence behavior, especially when clients have heterogeneous data, has not been fully understood yet. Existing theoretical results such as \cite{woodworth2020minibatch,glasgow2021sharp} predict that \fedavg's convergence is greatly affected by the data heterogeneity. When the local gradients on clients become different from each other (i.e., more data heterogeneity), \fedavg may require much more communication rounds to converge. These theoretical predictions match well with the observations on pathological datasets with artificially partitioned or synthetic data~\cite{hsu2019measuring,li2018federated}. However, on many real-world FL training tasks, \fedavg actually performs extremely well~\cite{mcmahan2016communication,charles2021large}, which appears to be unreasonable based on the existing theory. In fact, many advanced methods aimed at mitigating the negative effects of data heterogeneity performs similar to \fedavg in real-world FL training tasks. For example, \textsc{Scaffold}~\cite{karimireddy2019scaffold} needs much fewer communication rounds to converge than \fedavg in theory and when run on a synthetic dataset constructed to have large heterogeneity. However, \textsc{Scaffold} and \fedavg have roughly identical performance on many realistic federated datasets, see the experiments in \cite{reddi2020adaptive}. Thus, \emph{the negative theoretical results about \fedavg are mysteriously inconsistent with practical observations}.

\paragraph{Contributions.}
The significant gap between theory and practice motivates us to ask whether existing analyses of \fedavg are too pessimistic for practical applications. In particular, most previous works used the average difference between local gradients and the global gradient (i.e., gradient dissimilarity) as a measure of the influence of data heterogeneity. When gradient dissimilarity is zero, all clients have identical local objective functions, and hence, there is no heterogeneity. When gradient dissimilarity increases, the resulting error bound will become worse. However, in practice, it is possible that data heterogeneity yields large gradient dissimilarity but only has little influence on the actual convergence.

To illustrate this issue, in \Cref{sec:mismatch}, we consider a quadratic problem based on a statistical model, where all clients share the same labeling process but have different input distributions. In this simple problem, data heterogeneity does not adversely affect the convergence of \fedavg at all, though clients may have arbitrarily large gradient dissimilarity. This suggests that there may exist a significant mismatch between the level of gradient dissimilarity and the effect of data heterogeneity in the convergence of FedAvg. 

In order to capture the effects of data heterogeneity more accurately, we propose a new measure in \Cref{sec:proposed_measure}: \emph{\measure}, which is defined as the average (over clients) of the change in the local model after taking multiple local descent steps starting from the optimal point. In the quadratic problem of \Cref{sec:mismatch}, the \measure goes to zero almost surely as the number of clients goes to infinity, even if the gradient dissimilarity or the number of local steps is large. Similarly, we empirically show that on many realistic federated datasets, the \measure is very small and close to zero, while the upper bounds based on gradient dissimilarity are much larger. These observations explain the discrepancy between the existing theory and practice: due to a small (or roughly zero) \measure, data heterogeneity has very limited negative impact in practical applications but the effect on FedAvg is exaggerated in the existing analysis using gradient dissimilarity.
 
Moreover, in \Cref{sec:main_results}, we provide a new theoretical analysis of FedAvg for strongly convex loss functions, which explicitly utilizes \measure.
In the special case when \measure is zero, we prove that \fedavg has \emph{identical} convergence rate in homogeneous and heterogeneous data settings. As a consequence, one concludes that \fedavg is provably better than mini-batch SGD in regimes relevant to practical applications, despite strong data heterogeneity. Our analysis introduces some new proof techniques that extend the ideas in ~\cite{charles2021convergence,charles2022iterated,malinovskiy2020local}. These new proof techniques can be of independent interests.  

We would like to point out that our results provide alternative views that complement the existing literature. While it is true that on pathological datasets, larger data heterogeneity corresponds to larger gradient dissimilarity as well as a worse convergence, the key conclusion of this work is that there are many realistic datasets where data heterogeneity has little or even no negative impact on the convergence of FedAvg. As long as the \measure is small, \fedavg can be a reasonable algorithm to train a global model for all clients. If \measure is very large, then users may consider training personalized models instead of designing more sophisticated optimization procedures with global convergence guarantees. 

\section{Preliminaries and Related Work}
\paragraph{Problem Formulation.}
We consider total $M$ clients, where each client $c$ has a local objective function $F_c(\vw)$ defined on its local dataset $\gD_c$. The goal of FL training is to minimize a global objective function, defined as a weighted average over all clients:
\begin{align}
    F(\vw) = \sum_{c = 1}^M p_c F_c(\vw) =\E_c [F_c(\vw)] , \label{eqn:global_obj}
\end{align}
where $p_c$ is the relative weight for client $c$. For the ease of writing, in the rest of this paper, we will use $\E_c[\va_c] = \sum_{c=1}^M p_c \va_c$ to represent the weighted average over clients. In this paper, unless otherwise stated, we mainly focus on the setting where each local objective function is $L$-Lipschitz smooth, and $\mu$-strongly convex. Formally, there exists constants $L \geq \mu > 0$ such that
\begin{align}
    \mu \vecnorm{\vw - \vu} \leq \vecnorm{\nabla F_c(\vw) - \nabla F_c(\vu)} \leq L \vecnorm{\vw - \vu}.
\end{align}

\paragraph{Update Rule of \fedavg.}
\fedavg~\cite{mcmahan2016communication} is a popular algorithm to minimize \cref{eqn:global_obj} without the need of uploading raw training data. In each round of \fedavg, client $c$ performs $H$ steps of SGD from a global model $\vw$ to a local model $\vw^{(H)}_c$ with a local learning rate $\eta$. Then, at the server, the local model changes are aggregated to update the global model as follows:
\begin{align}
    \vw^{(t+1)} = \vw^{(t)} - \lr\E_c[\vw^{(t)} - \vw_c^{(t,H)}]. \label{eqn:update_rule_fedavg}
\end{align}
Here $\lr$ denotes the server learning rate, and superscript $t$ denotes the index of communication round. Unless otherwise stated, we assume that all clients participate into training at each round. 

\paragraph{Theoretical Analysis of \fedavg.}
When clients have homogeneous data, many works provided error upper bounds to guarantee the convergence of \fedavg (also called Local SGD) \cite{stich2018local,yu2018parallel,wang2018cooperative,zhou2017convergence,khaled2020tighter,li2019convergence}. In these papers, \fedavg was treated as a method to reduce the communication cost in distributed training. It has been shown that in the stochastic setting, using a proper $H>1$ in \fedavg will not negatively influence the dominant convergence rate. Hence \fedavg can save communication rounds compared to the algorithm where $H=1$. Later in \cite{woodworth2020local}, the authors compared \fedavg with the mini-batch SGD baseline, and showed that in certain regimes, \fedavg provably improves over mini-batch SGD. These upper bounds on \fedavg was later proved by \cite{glasgow2021sharp} to be tight and not improvable for general convex functions.

When clients have heterogeneous data, in order to analyze the convergence of \fedavg, it is common to make the following assumption to bound the difference among local gradients.
\begin{assump}[Gradient Dissimilarity]
There exists a positive constant $\zeta$ such that, $\forall \vw \in \R^d$,
\begin{align}
    \E_c\vecnorm{\nabla F_c(\vw) - \nabla F(\vw)} \leq \zeta^2. \label{eqn:grad_dissimilarity}
\end{align}
\end{assump}
This assumption first appeared in decentralized optimization literature~\cite{lian2017can,yuan2016convergence,assran2018stochastic}, and has been subsequently used in the analysis of \fedavg  \cite{yu2019linear,koloskova2020unified,khaled2020tighter,karimireddy2019scaffold,reddi2020adaptive,wang2020tackling,Wang2020SlowMo,haddadpour2019convergence}, since \fedavg can be considered as a special case of decentralized optimization algorithms~\cite{wang2018cooperative}.  Under the gradient dissimilarity assumption, \fedavg cannot outperform the simple mini-batch SGD baseline unless $\zeta$ is extremely small ($\zeta < 1/T$ where $T$ is the total communication rounds)~\cite{woodworth2020minibatch}; the deterministic version of \fedavg (i.e., Local GD) has even slower convergence rate than vanilla GD~\cite{khaled2020tighter}. Again, these upper bounds match a lower bound constructed in \cite{glasgow2021sharp} for general convex functions, suggesting that they are tight in the worst case. In this paper, we do not aim to improve these bounds, which are already tight. Instead, we argue that since the existing analyses only consider the worst case, they may be too pessimistic for practical applications.

Finally, we note that there is another line of works~\cite{malinovskiy2020local,charles2021convergence,charles2022iterated} using a different analysis technique from the above literature. They showed that \fedavg is equivalent to performing gradient descent on a surrogate loss function. However, so far this technique still has many limitations. It can only be applied to deterministic settings with quadratic (or a very special class) loss functions.

\section{Mismatch Between Gradient Dissimilarity and the Influence of Data Heterogeneity}\label{sec:mismatch}
In this section, we show that there is a significant mismatch between the gradient dissimilarity and the influence of data heterogeneity. We first review existing analysis techniques and discuss where the mismatch arises. Then, we consider a simple quadratic problem based on a realistic statistical model, where data heterogeneity does not affect the convergence rate even if the gradient dissimilarity is arbitrarily large. For the ease of discussion, in this section, we constrain ourselves to the deterministic setting where clients perform local GD updates instead of local SGD in each round.

\subsection{Overview of Existing Analysis Techniques based on Gradient Dissimilarity}\label{sec:overview_prev}

We first notice that most of previous works treat \fedavg as a biased mini-batch SGD algorithm, such as \cite{wang2020tackling,reddi2020adaptive,yang2020achieving,karimireddy2019scaffold}. In these theoretical works, the actual descent direction in \fedavg is defined as a pseudo-gradient as follows
\begin{align}
	\text{Pseudo-Gradient:} \quad \pseudograd_c(\vw) \triangleq \frac{1}{\eta H} (\vw - \vw_c^{(H)}) \label{eq:quadratic-pseudo-grad}
\end{align}
where $\vw_c^{(H)}$ denotes the locally trained model after performing $H$ steps of GD from $\vw$ using a local learning rate $\eta$. As a consequence, the update rule~\Cref{eqn:update_rule_fedavg} of \fedavg can be rewritten as follows:
\begin{align}
    \vw^{(t+1)} = \vw^{(t)} - \lr \eta H \nabla F(\vw^{(t)}) + \lr\eta H \E_c[B_c(\vw^{(t)})] \label{eqn:biased_sgd}
\end{align}
where $B_c(\vw^{(t)})$ denotes the gradient bias at client $c$, which is formally defined as
\begin{align}
    B_c(\vw^{(t)}) \triangleq \nabla F_c(\vw^{(t)}) - \pseudograd_c(\vw^{(t)}) = \frac{1}{H}\sum_{h=0}^{H-1}\parenth{\nabla F_c(\vw^{(t)}) - \nabla F_c(\vw_c^{(t,h)})}. \label{eqn:grad_bias}
\end{align}
Recall that $\vw_c^{(t,h)}$ represents the local model on client $c$ after performing $h$ steps of local updates from $\vw^{(t)}$. Using standard techniques, one can prove the following lemma for deterministic \fedavg.
\begin{lem}\label{lem:biased_gd}
Suppose \fedavg starts model training from $\vw^{(0)}$ and each local objective function is $L$-Lipschitz smooth and $\mu$-strongly convex. When the learning rates satisfy $\lr\eta H L \leq 1$, after total $T$ communication rounds, we have
\begin{align}
    F(\vw^{(T)}) - F(\vw^{*})
    \leq \parenth{1 - \lr\eta H \mu}^T(F(\vw^{(0)}) - F(\vw^*)) + \frac{1}{2\mu T}\sum_{t=0}^{T-1}\vecnorm{\E_c B_c(\vw^{(t)})} \label{eqn:previous_analysis_bnd}
\end{align}
where $\vw^* = \argmin_{\vw} F(\vw)$ denotes the optimum of the global objective function.
\end{lem}
\Cref{lem:biased_gd} shows that the convergence of \fedavg critically relies on the norm of the average gradient bias. So, one way to ensure the convergence is to to provide a uniform upper bound on this gradient bias term. To achieve this, previous works introduce the gradient dissimilarity assumption in the form of \Cref{eqn:grad_dissimilarity} to bound the second term in \Cref{eqn:previous_analysis_bnd} as follows:
\begin{align}
    \vecnorm{\E_c B_c(\vw^{(t)})}
    \leq& \E_c\vecnorm{B_c(\vw^{(t)})} \label{eqn:loose_bnd1} \\
    \leq& \beta \vecnorm{\nabla F(\vw^{(t)})} + \gamma \eta^2 L^2 H^2 \zeta^2 \label{eqn:loose_bnd2}
\end{align}
where $\beta < 1, \gamma$ are constants, and recall that $\zeta$ measures gradient dissimilarity. \Cref{eqn:loose_bnd1} comes from Jensen inequality, and \Cref{eqn:loose_bnd2} is due to lemmas in previous works~\cite{wang2020tackling,yang2020achieving}. After substituting \Cref{eqn:loose_bnd2} into \Cref{eqn:previous_analysis_bnd} and optimizing the learning rates, one can obtain a convergence rate $\gO(T^{-2})$ for \fedavg for strongly convex functions, which is substantially slower than vanilla GD's linear rate $\gO(\exp(-T))$.

\subsection{Motivating Example: Arbitrarily Large Gradient Dissimilarity but No Negative Impact}\label{sec:toy_example}

We argue that the upper bounds \Cref{eqn:loose_bnd1,eqn:loose_bnd2} might be too pessimistic for practical applications. Particularly, upper bound \cref{eqn:loose_bnd1} omits the correlations (i.e., covariance) across clients. It is possible that the average gradient bias is very small while their average $\ell_2$-norm is large. Furthermore, upper bound \Cref{eqn:loose_bnd2} suggests that the gradient bias quadratically increases with $H$. However, this has not been validated in practice. There may exist certain scenarios where the gradient bias increases slowly with $H$ or even does not depend on $H$.

Motivated by the above intuition, we construct a synthetic problem where a single global model can work reasonably well for all clients though they have heterogeneous data. We assume all clients have the same conditional probability $p_c(y|\vx)=p(y|\vx), \forall c$, where $\vx,y$ denote the input data and its label, respectively. In this case, clients still have heterogeneous data distributions, as they may have drastically different $p_c(\vx)$ and $p_c(\vx, y)$. However, clients contributions should not conflict with each other, as the learning algorithms tend to learn the same $p(y|\vx)$ on all clients.

Now let us study a concrete example. We assume that the the model is linear and the label of the $i$-th data sample on client $c$ is generated as follows:
\begin{align}
    y_{c,i} = \inprod{\vw^*}{\vx_{c,i}} + \epsilon_{c,i}
    \label{eqn:stats_model}
\end{align}
where $\vw^* \in \R^d$ denotes the optimal model, and $\epsilon_{c,i} \sim \mathcal{P}_\epsilon$ is a zero-mean random noise and independent from $\vx_{c,i}$ (this is a common assumption in statistical learning). We also assume that all $\|\vx_{c,i}\|$ and $\epsilon_{c,i}$ have bounded variance. 
Both $\vw^*$ and $\mathcal{P}_\epsilon$ are the same on all clients. That is, clients share the same label generation process (i.e., same conditional probability $p(y|\vx)$). Our goal is to find the optimal model $\vw^*$ given a large amount of clients with \emph{finite} data samples (common cross-device FL setting~\cite{kairouz2019advances}). One can define a quadratic loss function for each client as follows:
\begin{align}
    F_c(\vw) = \frac{1}{n}\sum_{i=1}^{n} \frac{1}{2}\parenth{y_{c,i} - \vw^\top \vx_{c,i}}^2 = \frac{1}{2}(\vw-\vw^*)\mA_c(\vw-\vw^*) - \vb_c^\top(\vw-\vw^*) + \text{const.}
\end{align}
where $\mA_c = \sum_{i=1}^{n} \vx_{c,i}\vx_{c,i}^\top/n$, $\vb_c = \sum_{i=1}^{n} \epsilon_{c,i}\vx_{c,i}/n$. The minimizer of local objective $F_c(\vw)$ is $\vw_c^* = \vw^* + \mA_c^{-1}\vb_c$, which is different from the global minimizer $\vw^*$ as $\vb_c \neq 0$.

\paragraph{Problems in Existing Analyses.}
Now let us check the gradient dissimilarity in this synthetic problem. At the optimal point $\vw^*$, according to the definition~\Cref{eqn:grad_dissimilarity}, we have 
\begin{align}
    \E_c\vecnorm{\nabla F_c(\vw^*) - \nabla F(\vw^*)} = \E_c\vecnorm{\nabla F_c(\vw^*)} = \E_c\vecnorm{\vb_c}. \label{eqn:quad_grad_dis}
\end{align}
Observe that $\epsilon$ can have extremely large variance such that the gradient dissimilarity bound $\zeta$ is arbitrarily large. As a result, existing analyses, which rely on the bounded gradient dissimilarity, may predict that \fedavg is much worse than its non-local counterparts. However, by simple manipulations on the update rule of \fedavg, one can easily prove the following theorem.
\begin{thm}\label{thm:quadratic_analysis}
Suppose that the weighting of the clients is uniform, and each client has a small finite amount of data. Under the problem setting of \eqref{eqn:stats_model}, the iterates of Local GD (i.e., deterministic version of \fedavg) satisfies the following equation almost surely as the number of clients goes to infinity:
\begin{align}
    \vw^{(T)} - \vw^* = \brackets{\E_c\brackets{\parenth{\mI - \eta\mA_c}^H}}^{T}(\vw^{(0)} - \vw^*).
\end{align}
\end{thm}
The proof is relegated to the Appendix. From \Cref{thm:quadratic_analysis}, it is clear that if the learning rate $\eta$ is properly set such that $(\mI - \eta\mA_c)$ is positive definite, then performing more local updates (larger $H$) will lead to faster \emph{linear} convergence rate $\gO(\exp(-T))$ to the global optimum $\vw^*$. That is, Local GD is strictly better than vanilla GD. However, previous works based on gradient dissimilarity will get a substantially slower rate of $\gO(1/T^2)$ as discussed in \Cref{sec:overview_prev}. In this example, while the gradient dissimilarity can be arbitrarily large, the data heterogeneity actually does not have any negative impacts. There is a significant mismatch between the level of gradient dissimilarity and the effects of data heterogeneity.

\GJ{Perhaps write the rate of convergence of local GD and vanilla GD here so that readers can compare and contrast with the rates mentioned at the end of Section 3.1} \ZX{+1. Or maybe just a comment on the best known FedAvg convergence so far.}
\JW{Added.}
\GJ{I also added $\gO(\exp(-T))$ above}

\section{Proposed Measure of Data Heterogeneity}\label{sec:proposed_measure}
In \Cref{sec:mismatch}, we show that there exist regimes where data heterogeneity does not have any negative impacts on the convergence of \fedavg, while the gradient dissimilarity among clients can be arbitrarily large. The bounded gradient dissimilarity assumption not only inaccurately describe the real effects of data heterogeneity but also may be violated in many scenarios. In order to address this problem, in this section, we propose a new measure of data heterogeneity and later empirically compare it with previous gradient dissimilarity on multiple real-world federated training tasks.

As we discussed in \Cref{sec:mismatch}, one problem about using the gradient dissimilarity is that it omits the possible correlations among clients. So instead of using the upper bound \Cref{eqn:loose_bnd1} which uses the expected norm, we propose to directly measure the effects of data heterogeneity through the norm of the expected value. Specifically, we use the gradient bias \Cref{eqn:grad_bias} at the optimal point $\vw^*$ to define \emph{\measure} and use it to measure data heterogeneity:
\begin{align}
	\text{\measure:} \quad \rho \triangleq \norm{\pseudograd(\vw^*) - \nabla F(\vw^*)} = \norm{\pseudograd(\vw^*)}, \label{eqn:client_consensus}
\end{align}
where $\pseudograd(\vw)=\E_{c} [\pseudograd_c(\vw)]$, and $\pseudograd_c(\vw)$ is pseudo-gradient defined in \cref{eq:quadratic-pseudo-grad}.
Using the definition of pseudo-gradient and mean value theorem, we can get the following lemma.

\begin{lem}\label{lem:average_drift}
	The \measure defined in~\Cref{eqn:client_consensus} is given by $\rho = \norm{\E_c\brackets{\mP_c \nabla F_c(\vw^*)}}$ where $\mP_c$s are polynomials of local Hessian matrices. When $F_c$ is quadratic, we have 
	\begin{align}
	    \mP_c = \frac{1}{H}\sum_{h=0}^{H-1}\brackets{\mI - (\mI - \eta \nabla^2 F_c(\vw^*))^h}.
	\end{align}
\end{lem}

Now, we are ready to compare the \measure and previous gradient dissimilarity in different settings. 

(1) Suppose all local objectives are identical, then $\nabla F_c(\vw^*)=\nabla F(\vw^*)=0$ and hence, \measure $\rho=0$. In this case, the gradient dissimilarity is zero as well.

(2) Suppose local objective functions are quadratic functions and have the same Hessian. In this case, for some matrix $\mP$, we have $\mP_c=\mP$ for all $c$, and $\rho=\|\E_c[\mP_c \nabla F_c(\vw^*)]\| = \|\mP \E_c[\nabla F_c(\vw^*)]\| = 0$. However, the gradient dissimilarity can be arbitrarily large.

(3) In our constructed problem in \Cref{sec:toy_example}, the \measure can be simplified to $\rho=\norm{\E_c[\mP_c \vb_c]}$. Due to the independence of $\epsilon_{c,i}$ and $\vx_{c,i}$, one can prove that for any choice of $H$, $\rho \rightarrow 0$ almost surely when the number of clients $M \to \infty$\footnote{When there are finite clients, we prove in the appendix that with high probability, \measure is $\rho=\gO(1/\sqrt{n M})$ while gradient dissimilarity is $\gO(1/\sqrt{n})$.}. This conclusion of zero \measure matches well with the linear convergence rate in \Cref{thm:quadratic_analysis}, which suggests that data heterogeneity has no negative impacts. However, the gradient dissimilarity can be arbitrarily large.

In other more general cases, it may not be always true that $\rho=0$ but we empirically show in \Cref{subsec:empirical} that in many practical applications, the \measure can be very small. We postulate that this explains why \fedavg performs unreasonably well in practice.

\subsection{Empirical Validation} \label{subsec:empirical}
In this subsection, we empirically show that our proposed measure of data heterogeneity, \measure, remains as a small value across multiple practical training tasks, suggesting that the effects of data heterogeneity can be very limited on these tasks. In contrast, both upper bounds \Cref{eqn:loose_bnd1,eqn:loose_bnd2} for the previous gradient similarity assumption are loose in practice, and hence, the resulting error bound underestimates the capacity of \fedavg.

In \Cref{fig:exp_loose_bounds2}, we first run mini-batch SGD on Federated EMNIST (FEMNIST)~\cite{mcmahan2016communication} and StackOverflow Next Word Prediciton datasets~\cite{reddi2019convergence} to obtain an approximation for the optimal model $\vw^*$. Then, we evaluate the \measure $\rho^2=\|\E_cB_c(\vw^*)\|^2$ and its upper bound $\E_c\|B_c(\vw^*)\|^2$ in  \Cref{eqn:loose_bnd1} on these datasets. If previous upper bounds based on the gradient dissimilarity are tight, then  $\rho^2$ and $\E_c\|B_c(\vw^*)\|^2$ should be close to each other and increase quadratically with the number of local steps $H$. However, it can be observed from \Cref{fig:exp_loose_bounds2a,fig:exp_loose_bounds2b}, the \measure  $\rho^2$ (red lines) nearly remains as zero and its upper bound (blue lines) slowly become larger when the number of local steps $H$ increases, suggesting previous upper bounds are loose. In addition, we construct a synthetic dataset based on our proposed statistical model~\Cref{eqn:stats_model}. As shown in \Cref{fig:exp_loose_bounds2c}, when the level of gradient dissimilarity changes, the upper bound \Cref{eqn:loose_bnd1} (blue lines) also experiences drastically change. Nonetheless, the \measure (red lines) stays as constant and is very close to zero. When clients perform stochastic local updates, we also get similar observations on StackOverflow, as shown in \Cref{fig:exp_loose_bounds2d}.

\begin{figure}[tb]
    \centering
    \begin{subfigure}{.33\textwidth}
    \includegraphics[width=\textwidth]{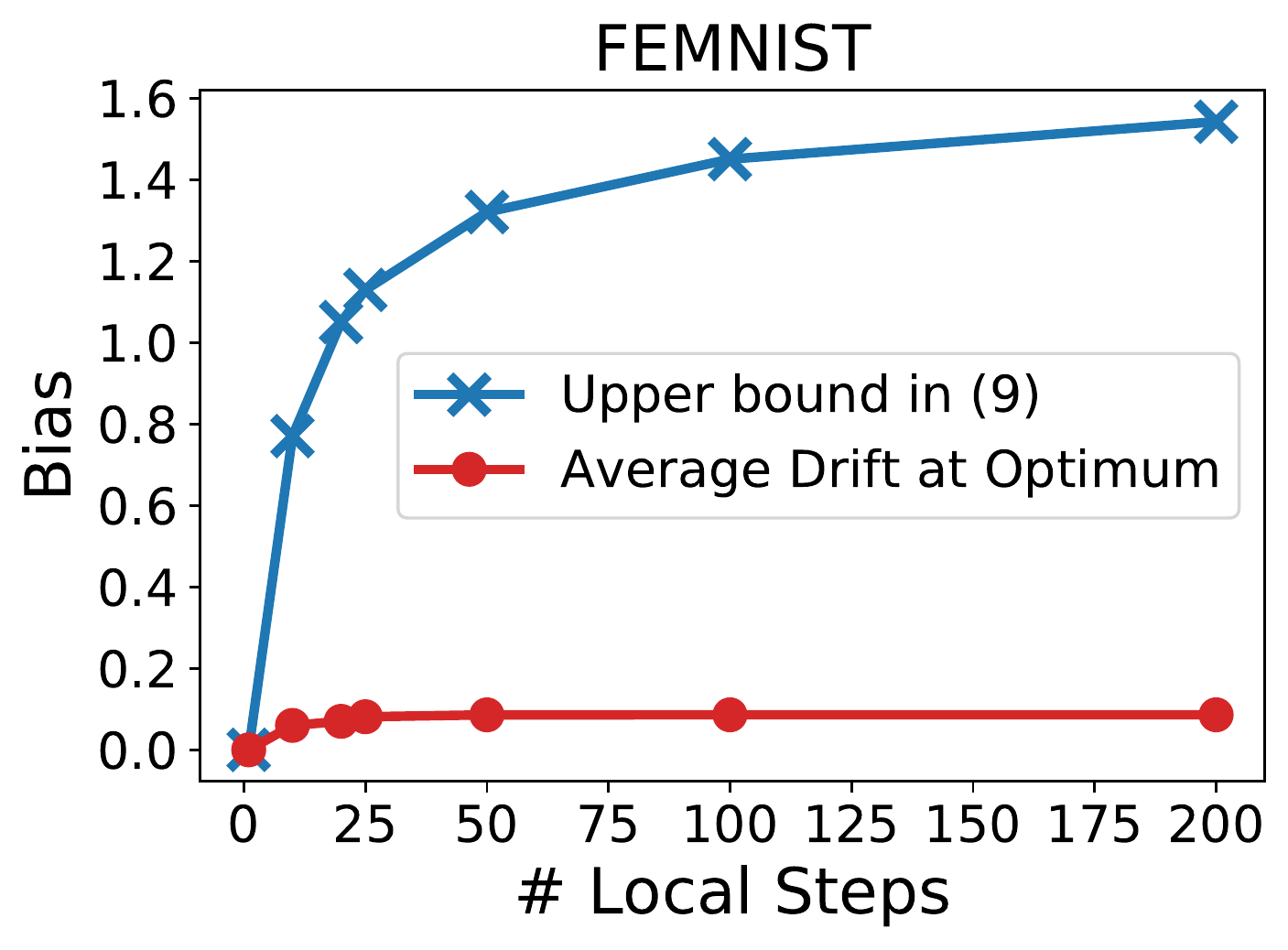}
    \caption{FEMNIST.}
    \label{fig:exp_loose_bounds2a}
    \end{subfigure}%
    ~
    \begin{subfigure}{.33\textwidth}
    \includegraphics[width=\textwidth]{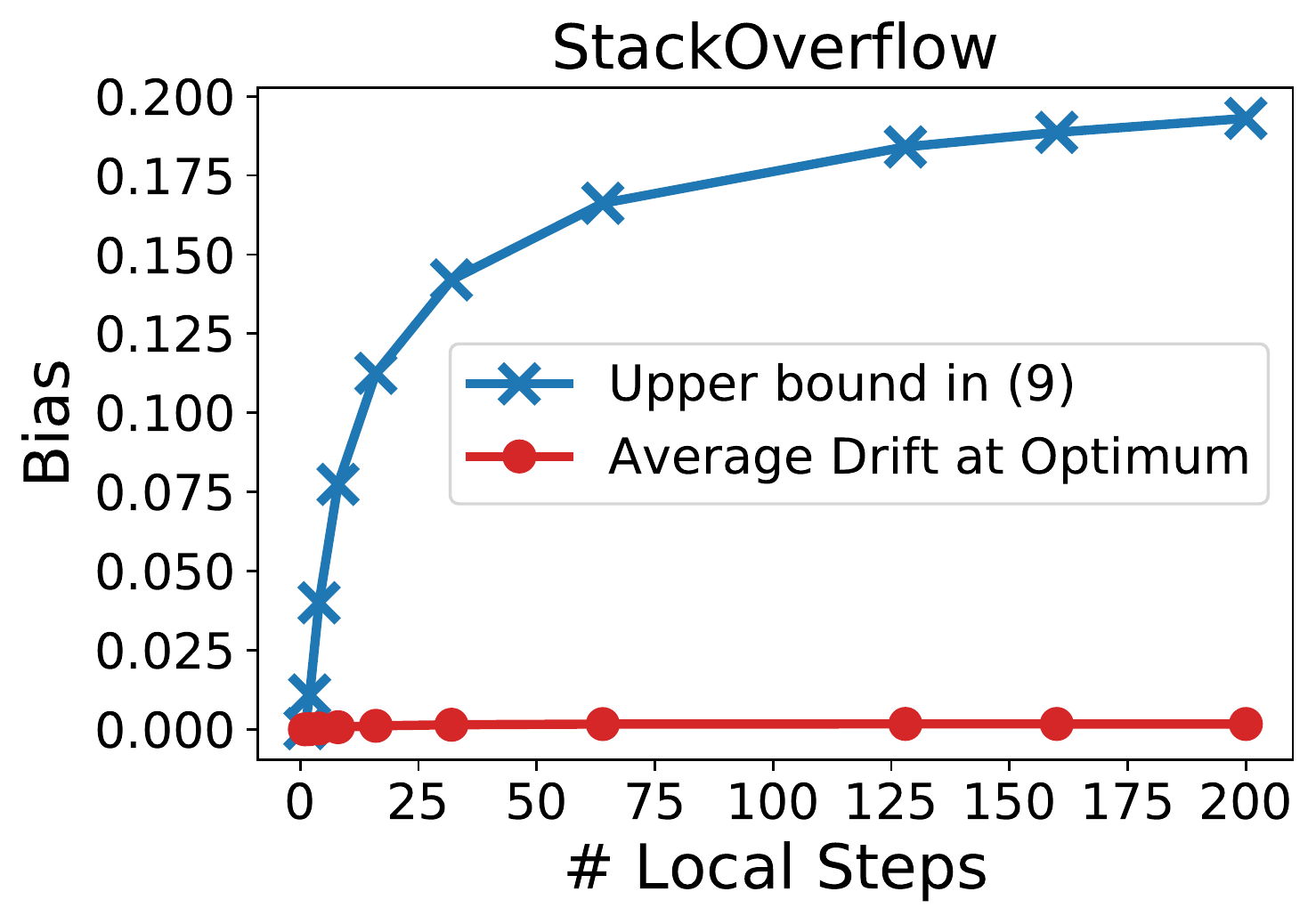}
    \caption{StackOverflow.}
    \label{fig:exp_loose_bounds2b}
    \end{subfigure}%
    
    \begin{subfigure}{.33\textwidth}
    \includegraphics[width=\textwidth]{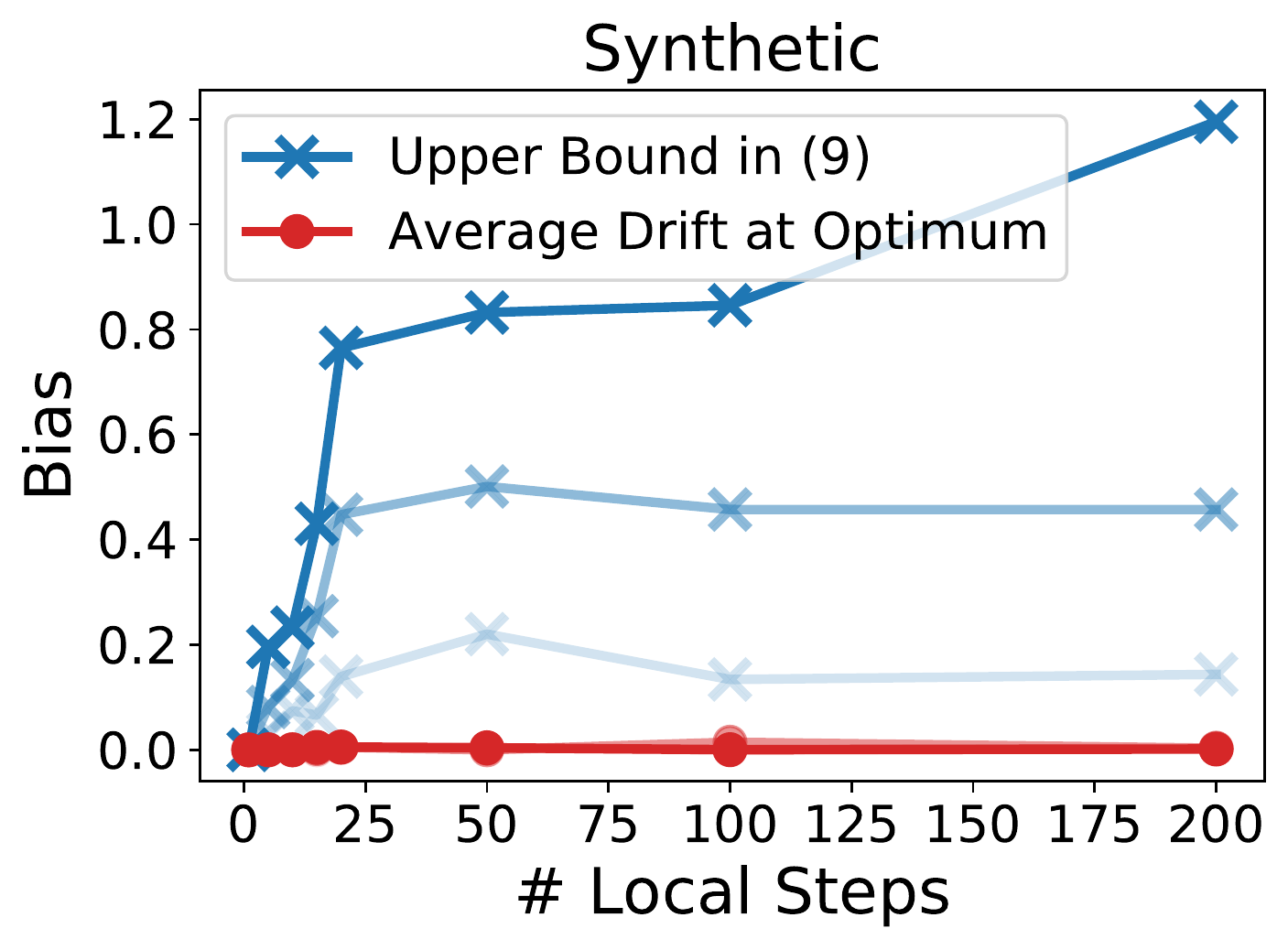}
    \caption{Synthetic Dataset.}
    \label{fig:exp_loose_bounds2c}
    \end{subfigure}%
    ~
    \begin{subfigure}{.33\textwidth}
    \includegraphics[width=\textwidth]{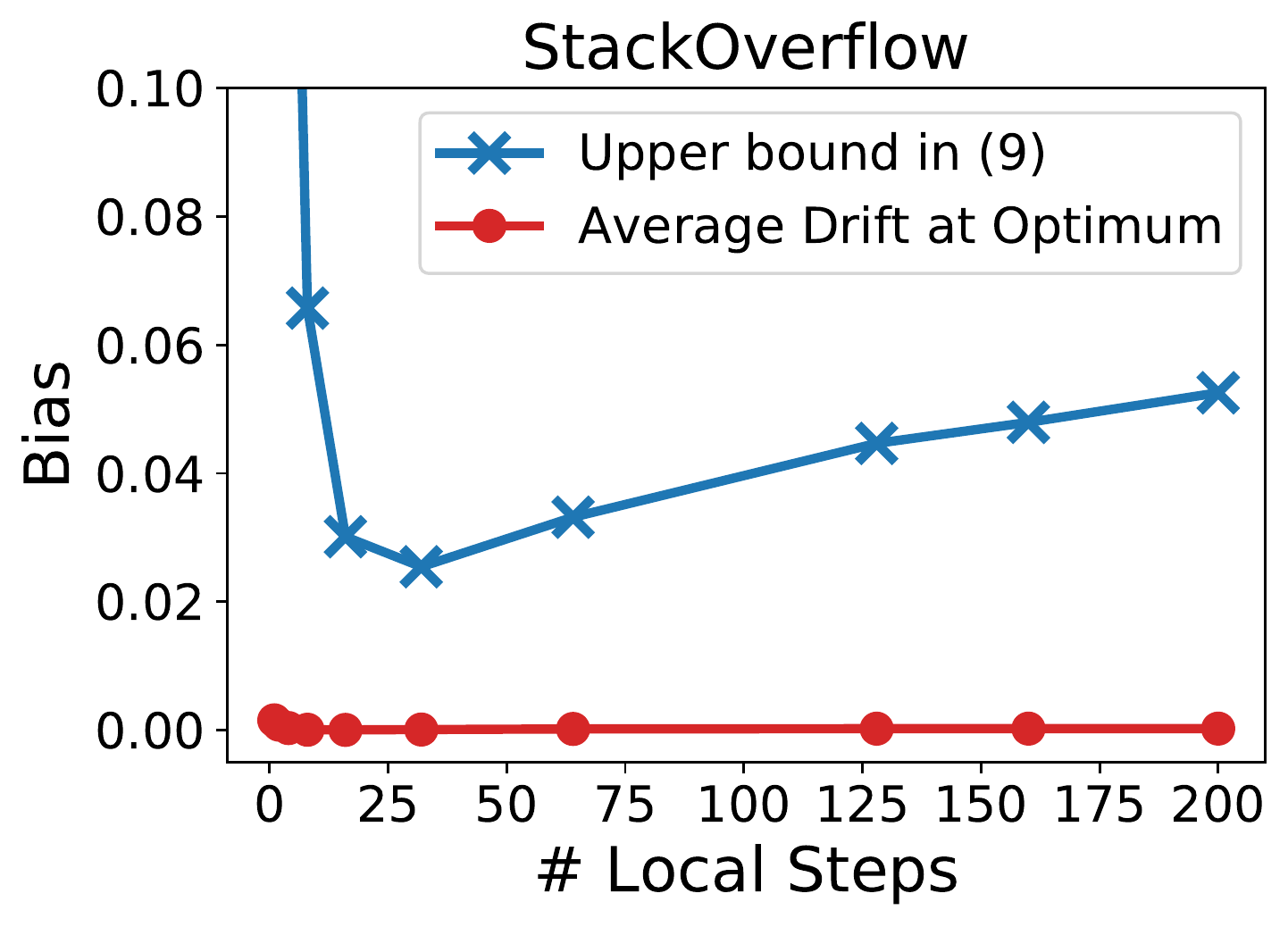}
    \caption{StackOverflow. Stochastic local updates.}
    \label{fig:exp_loose_bounds2d}
    \end{subfigure}%
    \caption{Gradient biases at the optimal point $\vw^*$ on three different datasets. We observe that the \measure (i.e, norm of the gradient bias at $\vw^*$, red line) nearly remain zero on all datasets but its upper bound \Cref{eqn:loose_bnd1} ($\E_c \| \nabla F_c(\vw^*) - \pseudograd_c(\vw^*) \|^2 $, blue lines) slowly become larger when $H$ increases. In (c), different lines in the same color correspond to different levels of data heterogeneity. \GJ{The names 'mean square' and 'squared mean' in the legend can be confusing to readers. Perhaps you can refer to the equations that define these terms?}\GJ{When you say different levels of gradient dissimilarity, is this a parameter that you choose, that is, the upper bound $\zeta$ on gradient dissimilarity? Or is it the actual gradient dissimilarity, the left side of (4)?}\GJ{If it is possible to drop one of these four plots, the remaining three will nicely fit in one row}}
    \label{fig:exp_loose_bounds2}
\end{figure}

Furthermore, we run \fedavg on FEMNIST dataset~\cite{mcmahan2016communication} following the same setup as \cite{reddi2020adaptive} and check the values of the gradient bias at several intermediate points on the optimization trajectory. For each point, we let clients perform local GD with the same local learning rate for multiple steps. As shown in \Cref{fig:exp_loose_bounds1}, we observe a significant gap between the quantity of interest (i.e., the norm of average gradient bias, red lines) and its upper bound \Cref{eqn:loose_bnd1} (i.e., the average of $\ell_2$ norms, blue lines). Especially, at the $50$-th and $100$-th rounds, the upper bound \Cref{eqn:loose_bnd1} is about $10$ times larger. In addition, note that both the norm of average gradient bias and the average of $\ell_2$ norms only increase slowly with the number of local steps $H$ and saturate after certain threshold. However, the upper bound \Cref{eqn:loose_bnd2} uses a quadratic function of $H$ to estimate them. All the above observations suggest that the upper bounds based on the gradient dissimilarity are pessimistic in practice.

\begin{figure}[!htb]
    \centering
    \begin{subfigure}{.3\textwidth}
    \includegraphics[width=\textwidth]{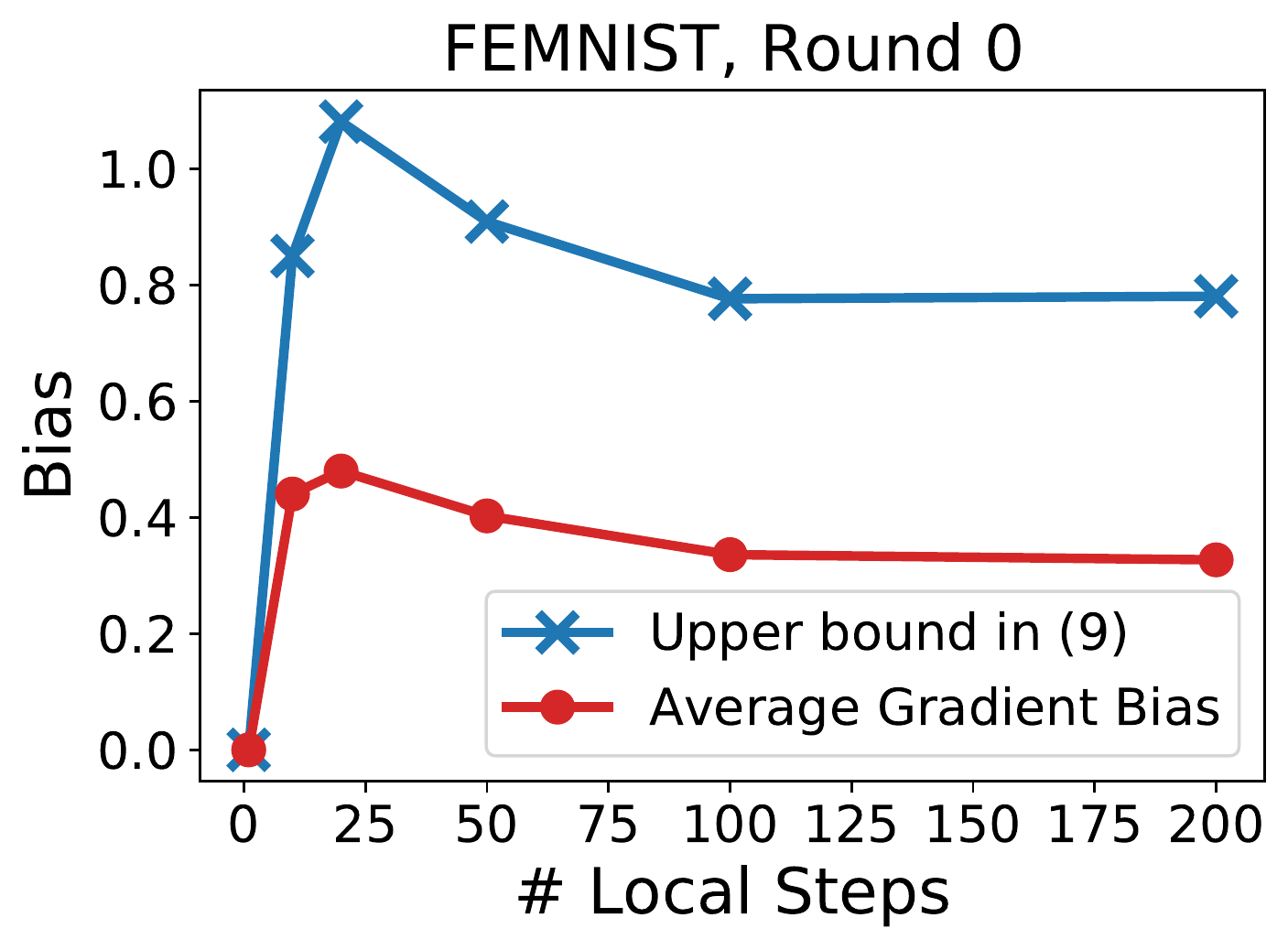}
    \caption{Round 0.}
    \end{subfigure}%
    ~
    \begin{subfigure}{.3\textwidth}
    \includegraphics[width=\textwidth]{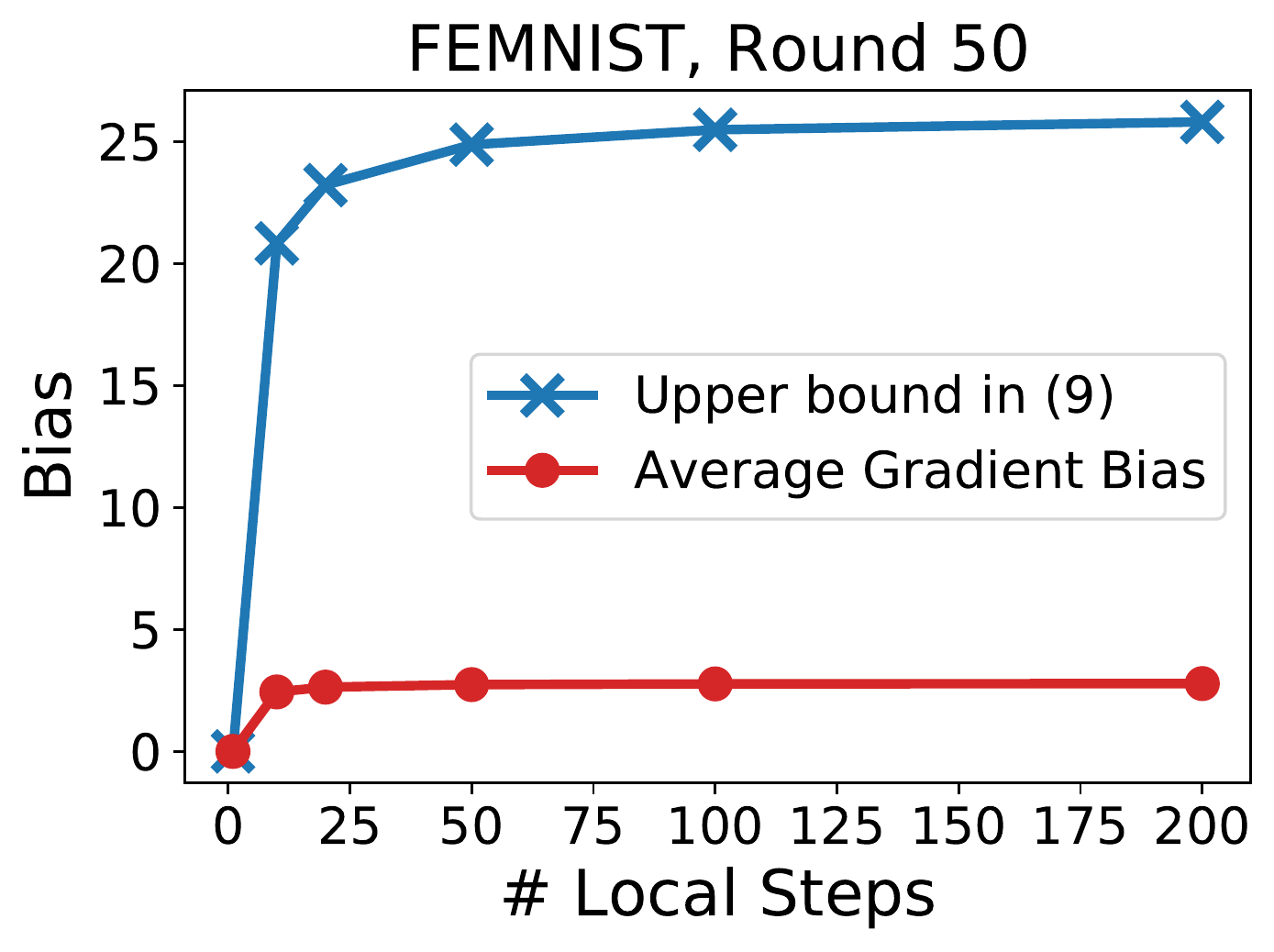}
    \caption{Round 50.}
    \end{subfigure}%
    ~
    \begin{subfigure}{.3\textwidth}
    \includegraphics[width=\textwidth]{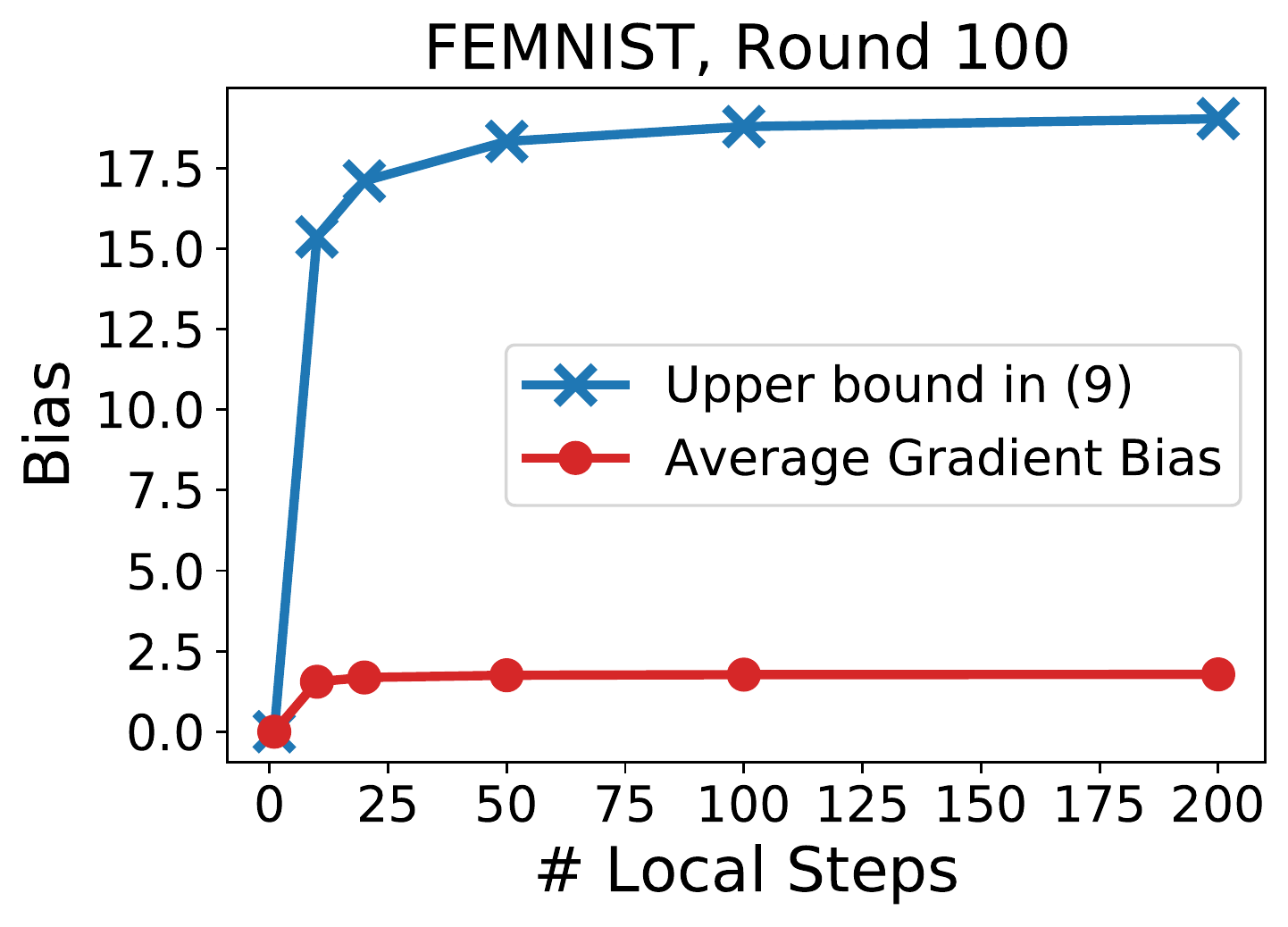}
    \caption{Round 100.}
    \end{subfigure}%
    ~
    \caption{Gradient biases at different points on a \fedavg's optimization trajectory. There is a significant gap between the average gradient bias (red line) and its upper bound (blue lines). Both of them increase and then saturate when increasing $H$.}
    \label{fig:exp_loose_bounds1}
\end{figure}

\section{Convergence Analysis} \label{sec:main_results}
In this section, we will provide a formal analysis (see \Cref{thm:one_round}) showing how the proposed measure of data heterogeneity, the \measure, influences the convergence of \fedavg. In particular, we provide convergence rates for strongly convex functions in \Cref{corollary:rate}, and for quadratic functions in \Cref{corollary:rate_quadratic}. The results show that there exist regimes ($\rho\simeq 0$) where \fedavg enjoys the same rates in both homogeneous and heterogeneous data settings.

We first generalize the pseudo gradient in \Cref{eq:quadratic-pseudo-grad} to stochastic settings.
\begin{defn}[\textbf{Pseudo-Gradient}]\label{defn:pseudo-gradient}
Given $\vw\in \R^d$, suppose $\vw_{c,\text{GD}}^{(H)}$ and $\vw_c^{(H)}$ denote the locally trained model on client $c$ after performing $H$ steps of GD and SGD using learning rate $\eta$, respectively. Then, we define
\begin{align}
    \text{Deterministic Pseudo-gradient:} \quad \pseudograd_c(\vw) &\triangleq \frac{1}{\eta H}(\vw - \vw_{c,\text{GD}}^{(H)}), \label{eqn:pseudo-gradient} \\
    \text{Stochastic Pseudo-gradient:} \quad \hat{\pseudograd}_c(\vw) &\triangleq \frac{1}{\eta H}(\vw - \vw_{c}^{(H)}). \label{eqn:stochastic-pseudo-gradient}
\end{align}
The average pseudo-gradients across all clients are defined as $\pseudograd = \E_c[\pseudograd_c]$ and $\hat{\pseudograd} = \E_c[\hat{\pseudograd_c}]$.
\end{defn}
Instead of treating the stochastic pseudo-gradient \Cref{eqn:stochastic-pseudo-gradient} as a biased estimation of the original gradient $\nabla F_c(\vw)$, we treat it as the biased version of \Cref{eqn:pseudo-gradient}. This new perspective allows us to extend previous analyses~\cite{charles2021convergence,charles2022iterated,malinovskiy2020local} to the stochastic setting\footnote{Since $\E[\hat{\pseudograd}(\vw)] \neq \pseudograd(\vw)$, the extension to the stochastic setting is non-trivial.} and explicitly feature the impacts of \measure in the final error bound. Also, following many standard SGD analyses~\cite{gower2019sgd,stich2019unified}, we assume the stochastic noise is upper bounded:
\begin{align}
	\E_\xi\vecnorm{\nabla F_c(\vw;\xi) - \nabla F_c(\vw)} \leq \sigma^2
\end{align}
where $\nabla F_c(\vw;\xi)$ denotes the stochastic gradient with respect to a random mini-batch $\xi$.

\subsection{Main Results}
Given the definition of pseudo-gradients, we then derive the following theorem.
\begin{thm}\label{thm:one_round}
When each local objective function is $L$-Lipschitz smooth and $\mu$-strongly convex and the learning rates satisfy $\lr\leq 1/4, \eta \leq \min\{1/L, 1/\mu H\}$, after $T$ rounds of \fedavg (Local SGD), we have
\begin{align}
\E\vecnorm{\vw^{(T)} - \vw^*}
\leq& (1-\frac{1}{2}\lr \eta H \mu)^T \vecnorm{\vw^{(0)} - \vw^*} + \frac{2\lr \eta H}{\mu} \max_{\vw} \Var[\hat{\pseudograd}(\vw)] \nonumber \\
    &+ \frac{20}{\mu^2}\max_{\vw} \E_c\vecnorm{\delta_c(\vw)} + \frac{20\rho^2}{\mu^2} \label{eqn:one_round}
\end{align}
where $\rho$ denotes the \measure, $\E[\cdot], \Var[\cdot]$ are taken with respect to random noise in stochastic local updates, and $\delta_c(\vw) = (\vw_{c,\text{GD}}^{(H)} - \E[\vw_c^{(H)}])/(\eta H)$ denotes the iterate bias between local GD and local SGD on client $c$.
\end{thm}
\paragraph{The Effects of Data Heterogeneity.}
As shown in \Cref{thm:one_round}, the effects of data heterogeneity can be fully captured by the \measure $\rho$. Instead of providing a uniform upper bound for $\norm{\pseudograd(\vw) - \nabla F(\vw)}$ like previous works, now we just need to ensure that $\rho=\norm{\pseudograd(\vw^*) - \nabla F(\vw^*)}=\|\pseudograd(\vw^*)\|$ is a small value. As we observed in our experimental results (see \Cref{fig:exp_loose_bounds1,fig:exp_loose_bounds2}), while $\norm{\pseudograd(\vw) - \nabla F(\vw)}$ can be large, the value of \measure $\rho=\|\pseudograd(\vw^*)\|$ is very close to zero on multiple realistic training tasks (FEMNIST, StackOverflow) for a reasonable range of $H$. In addition, in our proposed quadratic problem, one can prove that \fedavg implicitly ensures that for any $H$, we have $\rho \rightarrow 0$ almost surely as $M \rightarrow \infty$.
All these observations suggest that it is possible that heterogeneous data on clients do not have negative impacts on the convergence due to $\rho \approx 0$. However, this important regime has not been investigated before. 

\paragraph{The Effects of Stochastic Noise.}
From \Cref{thm:one_round}, we can observe that the stochastic noise during local updates influences the second and the third terms on the right-hand-side of \Cref{eqn:one_round}. The upper bounds of these two terms only depend on the dynamics of SGD algorithm, which has been well understood in literature. For example, in \cite{khaled2020tighter}, the authors show that $\E\|\vw^{(H)} - \E[\vw^{(H)}]\|^2 \leq 2\eta^2 H \sigma^2$. As a consequence, we directly obtain that
\begin{align}
	\Var[\hat{\pseudograd}(\vw)] = \frac{1}{M}\Var[\hat{\pseudograd}_c(\vw)] = \frac{1}{\eta^2 H^2 M} \E \vecnorm{\vw_c^{(H)} - \E[\vw_c^{(H)}]} \leq \frac{2\sigma^2}{MH}. \label{eqn:variance}
\end{align}
As for the iterate bias, one can obtain
\begin{align}
	\E_c\vecnorm{\delta_c(\vw)} \leq \eta^2 L^2 \sigma^2 (H-1), \label{eqn:iterate_bias}
\end{align}
the proof of which is provided in the Appendix. After substituting \Cref{eqn:variance,eqn:iterate_bias} into \Cref{eqn:one_round} and optimizing the learning rates, we can obtain the following convergence rate for \fedavg.
\begin{corollary}[\textbf{Convergence Rate for Strongly Convex Functions}]\label{corollary:rate}
Under the same setting as \Cref{thm:one_round}, when $\lr=1/4, \eta = \gO(1/\mu H T)$, the convergence rate of \fedavg is
\begin{align}
    \E\vecnorm{\vw^{(T)} - \vw^*} = \widetilde{\gO}\parenth{\frac{\sigma^2}{M H T} + \frac{\sigma^2}{H T^2} + \rho^2}. \label{eqn:final_rate}
\end{align}
If clients perform local GD instead of local SGD, then when $\eta = \min\{1/L, 1/(\mu H)\}$, we have
\begin{align}
    \vecnorm{\vw^{(T)} - \vw^*} = \gO\parenth{\exp\parenth{-\frac{T}{16\kappa}\min\{\kappa, H\}} + \rho^2} \label{eqn:final_rate_gd}
\end{align}
where $\kappa = L/\mu$ denotes the condition number.
\end{corollary}
In the special regime of $\rho \approx 0$ (as in our proposed quadratic problem when $M \rightarrow \infty$, and empirical observations on multiple datasets), \Cref{corollary:rate} states that, \fedavg has the same convergence rates in both homogeneous and heterogeneous data settings. As a result, data heterogeneity does not have negative impacts. However, in previous works based on gradient dissimilarity, even if $\rho=0$, there is an additional error. Moreover, in the deterministic setting, \fedavg's rate $\gO(\exp(-\mu HT/L))$ or $\gO(\exp(-T))$ is strictly better than GD's rate $\Omega(\exp(-\mu T / L))$ even in the presence of heterogeneous data. In contrast, as we discussed in \Cref{sec:overview_prev}, previous works can only get $\gO(1/T^2)$. A more detailed comparison with previous results is presented in \Cref{tab:comparison_v2}.

Furthermore, it is worth noting that the idea of treating local updates as pseudo-gradients has appeared in \cite{charles2021convergence,charles2022iterated}. However, their analysis techniques require additional specialized assumptions about pseudo-gradients that can be difficult to satisfy, and they only considered deterministic \fedavg. \Cref{thm:one_round,corollary:rate} developed new techniques to analyze pseudo-gradients under common loss function assumptions in optimization, and showed they can be applied to the stochastic setting.
\begin{table}[!th]
    \centering
    \begin{tabular}{l| l | l }\toprule
    Algorithm & Worst-case error & Comm. rounds to attain $\epsilon$ error (when $\rho=0$)\\ \midrule
    GD & $\exp(-T/\kappa)$  & $\gO(\kappa \log(1/\epsilon))$\\
    \fedavg \citep{koloskova2020unified}     & $\zeta^2/T^2$ & $\gO(1/\epsilon^2)$\\
    \fedavg \citep{woodworth2020minibatch}     & $1/(HT+H^2T^2) + \zeta^2/T^2$ & $\gO(1/\epsilon^2)$ \\
    \fedavg (Ours)     & $\exp(- T \min\{1, H/\kappa\}) + \rho^2$ & $\gO(\max\{1, \kappa/H\}\log(1/\epsilon))$\\ \bottomrule
    \end{tabular}
    \vspace{1em}
    \caption{Comparison with existing results for strongly convex objectives functions with deterministic local updates. In the table, the error is measured by the distance to the optimum $\|\vw - \vw^*\|^2$, and $\kappa=L/\mu$ is the condition number. Also, we omit logarithmic factors. Compared to previous results, we show that in the considered setting: (i) \fedavg enjoys linear convergence to the global optimum, and (ii) the multiple local steps of \fedavg mitigate the impact of of ill-conditioning (high condition number).}
    \label{tab:comparison_v2}
\end{table}

\subsection{Extensions}
Another benefit of using our analysis technique is that it allows us to easily incorporate additional assumptions. Below we provide several examples. 

(1) \emph{Client Sampling}: If we consider client sampling in \fedavg, then only the variance term in \Cref{eqn:one_round} will change and all other terms will not be affected at all. One can obtain new convergence guarantees by analyzing the variance of different sampling schemes and then simply substituting them into \Cref{eqn:one_round}. Standard techniques to analyze client sampling~\cite{yang2020achieving} can be directly applied.

(2) \emph{Alternative Heterogeneity Assumptions}: It is also possible to replace the data heterogeneity assumptions. Instead of setting the \measure $\rho$ to zero, one can also choose to apply the existing gradient dissimilarity technique to upper bound it. Alternatively, Hessian dissimilarity assumption (\cite{karimireddy2019scaffold}, $\|\nabla^2 F_c(\vw) - \nabla^2 F(\vw)\|\leq \beta$) may lead to an improved bound. We leave this extension to future work.

(3) \emph{Third-order Smoothness}: When the local objective functions satisfy third-order smoothness ($\norm{\nabla^2 F_c(\vw) - \nabla^2 F_c(\vu)} \leq Q \norm{\vw - \vu}$), the bound of the iterate bias $\delta(\vw)$ can be further improved while all other terms remain unchanged. According to \cite{glasgow2021sharp}, one can obtain $\norm{\delta(\vw)} \leq \frac{1}{4}\eta^2 H Q \sigma^2$. In the special case when local objective functions are quadratic, we have $Q=0$. That is, there is no iterate bias. As a consequence, the convergence rate of \fedavg can be significantly improved. Specifically, we have the following formal corollary.
\begin{corollary}[\textbf{Convergence Rate for Quadratic Functions}]\label{corollary:rate_quadratic}
Suppose each local objective functions is quadratic. Then, the convergence rate of \fedavg with a fixed learning rate is given as follows:
\begin{align}
    \E\vecnorm{\vw^{(T)} - \vw^*} = \widetilde{\gO}\parenth{\exp\parenth{-\frac{T}{16\kappa}\min\{\kappa, H\}} + \frac{\sigma^2}{M H T} + \rho^2}.
    \label{eqn:final_rate_quadratic}
\end{align}
\end{corollary}
\SK{Can we update the corollary to include dependence on $\rho$ similar to Corollary~\ref{corollary:rate}?}
Compared to \Cref{corollary:rate} for strongly covnex functions, \Cref{corollary:rate_quadratic} gets an improved convergence rate due to the nice properties of quadratic functions. In the special case when $\rho=0$, \fedavg is strictly better than mini-batch SGD, which is $\widetilde{\gO}(\exp(-T/\kappa) + \sigma^2/(M H T))$.

\section{Concluding Remarks}
In this paper, we aim at filling the gap between theory and practice about the convergence of the popular \fedavg algorithm. We found that previous analyses based on the bounded gradient dissimilarity assumption can be too pessimistic for practical applications. In order to accurately capture the effects data heterogeneity, we proposed a new measure, \measure, which remain close to zero across multiple real-world federated training tasks as well as our proposed quadratic problem, suggesting that data heterogeneity may have limited effect degrading FedAvg performance. At last, we provide a convergence analysis to illustrate how the \measure influences the convergence of \fedavg. The above new results can help to explain the empirical success of \fedavg. Future works may include extending the convergence analysis to general convex functions, and providing refined bounds for the \measure.

\bibliographystyle{abbrv}
\bibliography{iclr2022_conference}

\newpage
\appendix
\section{Details of the Synthetic Dataset}
The synthetic dataset we used in \Cref{fig:exp_loose_bounds2} is constructed based on the statistical model proposed in \Cref{sec:toy_example}. In particular, we set $\vx_{c,i}\in \R^{d}$ to be a random vector, where $d=30$ and each element is sampled from a uniform distribution $\gU(0,\nu_c)$, $\nu_c\sim\gU(0,5)$ is different on different clients. We assume there are total $M=100$ clients, each of which has $100$ data samples. Besides, we set $\epsilon_{c,i} \sim \gN(0, 0.09)$ and generate $\vw^*$ from a distribution $\gN(0, 1)$. As we shown in \Cref{eqn:quad_grad_dis}, the gradient dissimilarity tightly depends on the scale of $\epsilon_{c,i}$. So we control the gradient dissimilarity via changing the variance of $\epsilon_{c,i}$.

\section{Proof of \Cref{lem:biased_gd}}\label{sec:proof_lem1}
\begin{proof}
For the ease of writing, we define $\vw^+ = \vw^{(t+1)}$ and $\vw = \vw^{(t)}$. Then, according to Lipschitz Smoothness and $\lr\eta H L \leq 1$, we have
\begin{align}
    F(\vw^+) 
    \leq& F(\vw) - \lr\eta H \inprod{\nabla F(\vw)}{\pseudograd(\vw)} + \frac{\lr^2\eta^2 H^2 L}{2} \vecnorm{\pseudograd(\vw)} \\
    =& F(\vw) - \frac{\lr\eta H}{2}\vecnorm{\nabla F(\vw)} + \frac{\lr\eta H}{2}\vecnorm{\nabla F(\vw) - \pseudograd(\vw)} - \frac{\lr\eta H}{2}(1- \lr\eta H L)\vecnorm{\pseudograd(\vw)} \\
    \leq& F(\vw) - \frac{\lr\eta H}{2}\vecnorm{\nabla F(\vw)} + \frac{\lr\eta H}{2}\vecnorm{\nabla F(\vw) - \pseudograd(\vw)}.
\end{align}
Besides, note that strongly convexity yields $\vecnorm{\nabla F(\vw)} \geq 2\mu(F(\vw) - F(\vw^*))$. So we have
\begin{align}
    F(\vw^+) - F(\vw^*)
    \leq \parenth{1 - \lr\eta H \mu}(F(\vw) - F(\vw^*)) + \frac{\lr\eta H}{2}\vecnorm{\nabla F(\vw) - \pseudograd(\vw)}.
\end{align}
After total $T$ communication rounds, it follows that
\begin{align}
    F(\vw^{(T)}) - F(\vw^{(*)})
    \leq \parenth{1 - \lr\eta H \mu}^T(F(\vw^{(0)}) - F(\vw^*)) + \frac{1}{2\mu T}\sum_{t=0}^{T-1}\vecnorm{\nabla F(\vw^{(t)}) - \pseudograd(\vw^{(t)})}.
\end{align}
\end{proof}

\section{Proof of \Cref{thm:quadratic_analysis}}\label{sec:prof_lem1}
\begin{proof}
According to the local update rule, we have
\begin{align}
    \vw_c^{(t,h+1)} 
    =& \vw_c^{(t,h)} - \eta \nabla F_c(\vw_c^{(t,h)}) \\
    =& \vw_c^{(t,h)} - \eta \brackets{\mA_c(\vw_c^{(t,h)} - \vw^*) - \vb_c} \\
    =& \parenth{\mI - \eta\mA_c}\vw_c^{(t,h)} + \eta \parenth{\mA_c\vw^* + \vb_c}.
\end{align}
Subtracting $\vw_c^* = \vw^* + \mA_c^{-1}\vb_c$ on both sides, it follows that
\begin{align}
    \vw_c^{(t,h+1)} - \vw_c^* 
    =& \parenth{\mI - \eta \mA_c}\parenth{\vw_c^{(t,h)} - \vw_c^*} \\
    =& \parenth{\mI - \eta \mA_c}^{h+1} \parenth{\vw^{(t)} - \vw_c^*}.
\end{align}
Setting $h=H$, we have $\vw_c^{(t,H)} = (\mI - \eta\mA_c)^H (\vw^{(t)} - \vw_c^*)+\vw_c^*$. Recall the definition of pseudo-gradient~\Cref{eq:quadratic-pseudo-grad}, we get
\begin{align}
    \pseudograd_c(\vw^{(t)})
    =& \frac{1}{\eta H} (\vw^{(t)} - \vw^{(t,H)}) \\
    =& \frac{1}{\eta H}\brackets{\mI - (\mI - \eta \mA_c)^H}  (\vw^{(t)} - \vw^*_c).
\end{align}
According to the global update rule of \fedavg, one can obtain that
\begin{align}
    \vw^{(t+1)}
    =& \vw^{(t)} - \lr\eta H\E_c\pseudograd_c(\vw^{(t)}) \\
    =& \vw^{(t)} - \lr\E_c\brackets{(\mI - (\mI - \eta\mA_c)^{H})\parenth{\vw^{(t)} - \vw_c^*}} \\
    =& \vw^{(t)} - \lr\E_c\brackets{(\mI - (\mI - \eta\mA_c)^{H})\parenth{\vw^{(t)} - \vw^*}} \nonumber \\
        &- \lr\E_c\brackets{(\mI - (\mI - \eta\mA_c)^{H})\parenth{\vw^*-\vw_c^*}}.
\end{align}
Subtracting $\vw^*$ on both sides and setting $\lr=1$, we have
\begin{align}
    \vw^{(t+1)} - \vw^*
    =& \E_c\brackets{\parenth{\mI - \eta\mA_c}^H}\parenth{\vw^{(t)} - \vw^*} - \E_c\brackets{\underbrace{(\mI - (\mI - \eta\mA_c)^{H})\parenth{\vw^*-\vw_c^*}}_{\eta H \pseudograd_c(\vw^*)}} \\
    =& \E_c\brackets{\parenth{\mI - \eta\mA_c}^H}\parenth{\vw^{(t)} - \vw^*} - \eta H \pseudograd(\vw^*)
\end{align}
where $\pseudograd(\vw^*) = \E_c\pseudograd_c(\vw^*)$. Assume that $\pseudograd(\vw^*)=0$, then we have
\begin{align}
    \vw^{(t+1)} - \vw^* = \brackets{\E_c\brackets{\parenth{\mI - \eta\mA_c}^H}}^{t+1}(\vw^{(0)} - \vw^*),
\end{align}
which proves the desired result.
  
In the following, we are going to prove $\pseudograd(\vw^*)=0$ almost surely as $M \rightarrow \infty$ on this synthetic problem. According to the definition of $\mA_c, \vb_c$, we obtain that
\begin{align}
    \pseudograd(\vw^*)
    =& \E_c\brackets{(\mI - (\mI - \eta\mA_c)^{H})\mA_c^{-1}\vb_c} \\
    =& \E_c\brackets{\underbrace{(\mI - (\mI - \eta\mA_c)^{H})\mA_c^{-1}\frac{1}{n}\sum_{i=1}^{n}\vx_{c,i}\epsilon_{c,i}}_{\xi_c}} .
\end{align}
Since the noise $\epsilon_{c,\cdot}$ are independent of $\vx_{c,\cdot}$, $\xi_c$ is a zero-mean random variable that depends on client $c$. Since we have assumed that all $\|\vx_{c,i}\|$ and $\epsilon_{c,i}$ have bounded variance, we know that $\E [\xi_c^2] < \infty$. Since we have a uniform weighting on the $M$ clients, it follows $\E_c [\xi_c] = O(1/\sqrt{M})$ with probability $1 - o_M(1)$, and as $M \to \infty$, we have $\E_c [\xi_c] \to 0$ almost surely. 

\end{proof}

\section{Proof of \Cref{lem:average_drift}}
Before diving into the proof details, we would like to first introduce a lemma, which will be frequently used in the subsequent sections.
\begin{lem}[Mean Value Theorem]\label{lem:mvt}
Suppose function $F$ is twice differentiable, then
\begin{align}
    \nabla F_c(\vw) - \nabla F_c(\vu) = \mA_c(\vw, \vu)\cdot (\vw-\vu)
\end{align}
where $\mA_c(\vw, \vu)=\int_0^1 \nabla^2 F_c(\vu + s(\vw - \vu)) \text{d}s$. If $\mu \preceq \nabla^2 F_c \preceq L$, then it follows that $\mu \preceq \mA_c(\vw,\vu)\preceq L$ for any $\vw,\vu$.
\end{lem}
According to \Cref{lem:mvt}, for any local gradient at client $c$, we have
\begin{align}
    \nabla F_c(\vw^{(t,h+1)})
    =& \nabla F_c(\vw^{(t,h)}) + \mB_c^{(t,h)}(\vw^{(t,h+1)} - \vw^{(t,h)}) \\
    =& (\mI - \eta \mB_c^{(t,h)}) \nabla F_c(\vw^{(t,h)}) \\
    =& \prod_{s=0}^{h} (\mI - \eta\mB_c^{(t,s)}) \nabla F_c(\vw^{(t)})
\end{align}
where $\mB_c^{(t,h)}$ is an integral of the Hessian matrix. Then, recall the definition of pseudo-gradient~\Cref{eq:quadratic-pseudo-grad}, we have
\begin{align}
    \pseudograd_c(\vw^{(t)})
    =& \frac{1}{H}\sum_{h=0}^{H-1} \nabla F_c(\vw^{(t,h)}) \\
    =& \brackets{\frac{\mI}{H}+\frac{1}{H}\sum_{h=1}^{H-1}\prod_{s=0}^{h-1}(\mI - \eta\mB_c^{(t,s)})} \nabla F_c(\vw^{(t)}). \label{eqn:linear_transform}
\end{align}
From \Cref{eqn:linear_transform}, one can observe that the pseudo-gradient is a linear transformation of the local gradient at the starting point. As a result, we can write the \measure as follows:
\begin{align}
    \rho 
    =& \norm{\E_c[\pseudograd_c(\vw^*) - \nabla F_c(\vw^*)]} \\
    =& \norm{\E_c \underbrace{\brackets{\mI-\frac{\mI}{H}-\frac{1}{H}\sum_{h=1}^{H-1}\prod_{s=0}^{h-1}(\mI - \eta\mB_c^{(*,s)})}}_{\mP_c} \nabla F_c(\vw^*)}.
\end{align}
When local objective functions are quadratic, we have $\mB_c^{(*,s)} = \nabla^2 F_c(\vw^*)$ for any $s$. Accordingly, it follows that
\begin{align}
    \mP_c 
    =& \mI - \frac{\mI}{H} - \frac{1}{H}\sum_{h=1}^{H-1}\prod_{s=0}^{h-1}(\mI - \eta\mB_c^{(*,s)}) \\
    =& \mI - \frac{1}{H}\sum_{h=0}^{H-1} (\mI - \eta \nabla^2 F_c(\vw^*))^h \\
    =& \frac{1}{H}\sum_{h=0}^{H-1} \brackets{ \mI - (\mI - \eta \nabla^2 F_c(\vw^*))^h}.
\end{align}
Here we complete the proof.

\section{Proof of \Cref{thm:one_round}}
\subsection{Preliminaries}
In this subsection, we will first introduce several useful lemmas, which relate to the properties of the deterministic pseudo-gradients.
\begin{lem}[\textbf{Convexity, Smoothness \& Co-coercivity}]\label{lem:properties}
When each local objective function $F_c$ is $L$-Lipschitz smooth and $\mu$-strongly convex, for any $\vw, \vu \in \R^d$, we have
\begin{align}
    \effmu \vecnorm{\vw - \vu} \leq& \inprod{\pseudograd(\vw) - \pseudograd(\vu)}{\vw - \vu} \leq \effLip \vecnorm{\vw - \vu}, \label{eqn:pseudo-grad-properties1}\\
    \vecnorm{\pseudograd(\vw) - \pseudograd(\vu)} \leq& \effLip\inprod{\pseudograd(\vw) - \pseudograd(\vu)}{\vw - \vu}, \label{eqn:pseudo-grad-properties2}
\end{align}
where $\effmu = [1-(1-\eta \mu)^H]/(\eta H)$ and $\effLip = [1+(1-\eta \mu)^H]/(\eta H)$.
\end{lem}
\begin{proof}
Let us first focus on the pseudo-gradient on a specific client $c$. According to the definition of pseudo-gradient, we have
\begin{align}
    \eta H [\pseudograd_c(\vw) - \pseudograd_c(\vu)] 
    =& \vw - \vu - (\vw_c^{(H)} - \vu_c^{(H)}) \\
    =& \vw - \vu \nonumber \\
     & - [\vw_c^{(H-1)} - \vu_c^{(H-1)} - \eta(\nabla F_c(\vw_c^{(H-1)}) - \nabla F_c(\vu_c^{(H-1)}))] \label{lem2:1}\\
    =& \vw - \vu - (I - \eta \mD_c^{(H-1)})(\vw_c^{(H-1)} - \vu_c^{(H-1)}) \label{lem2:2}
\end{align}
where \Cref{lem2:2} follows from \Cref{lem:mvt} and $\mD_c$ is a symmetric matrix satisfying $\mu \preceq \mD_c \preceq L$. Repeating the above procedure, we can obtain that
\begin{align}
    \eta H [\pseudograd_c(\vw) - \pseudograd_c(\vu)] 
    =& \vw - \vu - \prod_{k=0}^{H-1}(\mI - \eta \mD_c^{(k)})(\vw - \vu) \\
    =& \brackets{\mI - \prod_{k=0}^{H-1}(\mI - \eta \mD_c^{(k)})}(\vw - \vu).
\end{align}
As a consequence, we have
\begin{align}
    \eta H \inprod{\pseudograd_c(\vw) - \pseudograd_c(\vu)}{\vw - \vu}
    =& \vecnorm{\vw - \vu} - \inprod{\prod_{k=0}^{H-1}(\mI - \eta \mD_c^{(k)})(\vw - \vu)}{\vw - \vu}. 
\end{align}
Note that, due to Cauchy–Schwarz inequality,
\begin{align}
    \left |\inprod{\prod_{k=0}^{H-1}(\mI - \eta \mD_c^{(k)})(\vw - \vu)}{\vw - \vu} \right |
    \leq& \prod_{k=0}^{H-1}\norm{\mI - \eta \mD_c^{(k)}} \vecnorm{\vw - \vu} \\
    \leq& (1-\eta \mu)^H \vecnorm{\vw - \vu}.
\end{align}
That is,
\begin{align}
    -(1-\eta\mu)^H \vecnorm{\vw - \vu}
    \leq \inprod{\prod_{k=0}^{H-1}(\mI - \eta \mD_c^{(k)})(\vw - \vu)}{\vw - \vu} 
    \leq (1-\eta\mu)^H \vecnorm{\vw - \vu}.
\end{align}
It follows that,
\begin{align}
    \frac{1-(1-\eta\mu)^H}{\eta H}\vecnorm{\vw - \vu} 
    \leq \inprod{\pseudograd_c(\vw) - \pseudograd_c(\vu)}{\vw - \vu}
    \leq \frac{1+(1-\eta\mu)^H}{\eta H}\vecnorm{\vw - \vu}. 
\end{align}
Taking the average over all clients, we complete the proof of \Cref{eqn:pseudo-grad-properties1}.

Next, we are going to prove \Cref{eqn:pseudo-grad-properties2}. Note that
\begin{align}
    \norm{\vw_c^{(H)} - \vu_c^{(H)}}
    =& \norm{\prod_{k=0}^{H-1}(\mI - \eta \mD_c^{(k)})(\vw - \vu)} \\
    \leq& \prod_{k=0}^{H-1}\norm{\mI - \eta \mD_c^{(k)}}\norm{\vw - \vu} \\
    \leq& (1-\eta\mu)^H \norm{\vw - \vu}.
\end{align}
As a result, we have
\begin{align}
    \vecnorm{\vw^{(H)} - \vu^{(H)}}
    =\vecnorm{\E_c\vw_c^{(H)} - \E_c\vu_c^{(H)}}
    \leq& \E_c\vecnorm{\vw_c^{(H)} - \vu_c^{(H)}} \\
    \leq& [(1-\eta \mu)^H]^2 \vecnorm{\vw - \vu}.
\end{align}
Then, according to the definition of pseudo-gradients, one can obtain
\begin{align}
    \eta^2H^2\vecnorm{\pseudograd(\vw) - \pseudograd(\vu)}
    =& \vecnorm{\vw - \vu - \vw^{(H)} + \vu^{(H)}} \\
    =& \vecnorm{\vw - \vu} + \vecnorm{\vw^{(H)} - \vu^{(H)}} - 2\inprod{\vw - \vu}{\vw^{(H)} - \vu^{(H)}} \\
    \leq& [1+(1-\eta\mu)^H]\vecnorm{\vw - \vu} - 2\inprod{\vw - \vu}{\vw^{(H)} - \vu^{(H)}} \nonumber \\
        &- (1-\eta\mu)^H\brackets{1-(1-\eta\mu)^H}\vecnorm{\vw - \vu} \\
    =& [1+(1-\eta\mu)^H]\brackets{\vecnorm{\vw - \vu} - \inprod{\vw - \vu}{\vw^{(H)} - \vu^{(H)}}} \nonumber \\
        &- (1- (1-\eta\mu)^H)\inprod{\vw - \vu}{\vw^{(H)} - \vu^{(H)}} \nonumber \\
        &- (1-\eta\mu)^H\brackets{1-(1-\eta\mu)^H}\vecnorm{\vw - \vu} \\
    =&[1+(1-\eta\mu)^H]\brackets{\vecnorm{\vw - \vu} - \inprod{\vw - \vu}{\vw^{(H)} - \vu^{(H)}}} \nonumber \\
        &+ (1- (1-\eta\mu)^H)\brackets{\vecnorm{\vw - \vu} - \inprod{\vw - \vu}{\vw^{(H)} - \vu^{(H)}}} \nonumber \\
        &- \brackets{1+ (1-\eta\mu)^H}\brackets{1-(1-\eta\mu)^H}\vecnorm{\vw - \vu}
\end{align}
Note that $\eta H \inprod{\vw - \vu}{\pseudograd(\vw) - \pseudograd(\vu)} = \vecnorm{\vw - \vu} -\inprod{\vw - \vu}{\vw^{(H)} - \vu^{(H)}}$ and $\eta H \effLip = 1 + (1-\eta\mu)^H$, $\eta H \effmu = 1 - (1-\eta\mu)^H$, we have
\begin{align}
    \vecnorm{\pseudograd(\vw) - \pseudograd(\vu)}
    \leq& (\effLip + \effmu)\inprod{\vw - \vu}{\pseudograd(\vw) - \pseudograd(\vu)} - \effLip\effmu\vecnorm{\vw - \vu} \\
    =& \effLip\inprod{\vw - \vu}{\pseudograd(\vw) - \pseudograd(\vu)} \nonumber \\
        & -\effmu \brackets{\effLip\vecnorm{\vw - \vu} - \inprod{\vw - \vu}{\pseudograd(\vw) - \pseudograd(\vu)}} \\
    \leq& \effLip\inprod{\vw - \vu}{\pseudograd(\vw) - \pseudograd(\vu)}
\end{align}
where the last inequality is due to the smoothness of the pseudo-gradient~\Cref{eqn:pseudo-grad-properties1}.
\end{proof}

\subsection{Main Proof}
In the analysis below, we first analyze the training progress within one round. Suppose the current global model is $\vw$ and the next round's global model is $\vw^+$. Without otherwise stated, the expectation $\E$ and variance $\Var$ are conditioned on the current global model $\vw$. For the ease of writing, we define effective learning rate $\efflr = \lr \eta H$.

First, according to the update rule of \fedavg, we have
\begin{align}
    \E\vecnorm{\vw^+ - \vw^*}
    =& \E\vecnorm{\vw - \efflr \hat{\pseudograd}(\vw) - \vw^*} \\
    =& \vecnorm{\vw - \efflr\E[\hat{\pseudograd}(\vw)] - \vw^*} + \efflr^2\Var[\hat{\pseudograd}(\vw)] \\
    =& \vecnorm{\vw - \vw^*} + \efflr^2 \vecnorm{\E[\hat{\pseudograd}(\vw)]} - 2\efflr\inprod{\vw-\vw^*}{\E[\hat{\pseudograd}(\vw)]} + \efflr^2\Var[\hat{\pseudograd}(\vw)]
\end{align}
Also, note that $\E[\hat{\pseudograd}(\vw)] = \pseudograd(\vw) - \pseudograd(\vw^*) + \pseudograd(\vw^*) + \delta(\vw)$. So one can obtain
\begin{align}
    \E\vecnorm{\vw^+ - \vw^*}
    \leq& \vecnorm{\vw - \vw^*} + 2\efflr^2 \vecnorm{\pseudograd(\vw) - \pseudograd(\vw^*)} - 2\efflr\inprod{\vw-\vw^*}{\pseudograd(\vw) - \pseudograd(\vw^*)} \nonumber \\
        & +2\efflr^2\vecnorm{\pseudograd(\vw^*)+\delta(\vw)} -2\efflr\inprod{\vw - \vw^*}{\pseudograd(\vw^*)+\delta(\vw)} + \efflr^2\Var[\hat{\pseudograd}(\vw)] \\
    \leq& (1-\efflr \effmu)\vecnorm{\vw - \vw^*} + 2\efflr^2\vecnorm{\pseudograd(\vw) - \pseudograd(\vw^*)} - \efflr\inprod{\vw-\vw^*}{\pseudograd(\vw) - \pseudograd(\vw^*)} \nonumber \\
        & +2\efflr^2\vecnorm{\pseudograd(\vw^*)+\delta(\vw)} -2\efflr\inprod{\vw - \vw^*}{\pseudograd(\vw^*)+\delta(\vw)} + \efflr^2\Var[\hat{\pseudograd}(\vw)] \label{eqn:mainproof_1}
\end{align}
where the first inequality uses the fact $\vecnorm{a+b} \leq 2\vecnorm{a} + 2\vecnorm{b}$, and the second inequality comes from the strongly-convexity of the pseudo-gradient. Now let us check the value of the following terms:
\begin{align}
    T_1 
    =& 2\efflr\vecnorm{\pseudograd(\vw) - \pseudograd(\vw^*)} - \inprod{\vw-\vw^*}{\pseudograd(\vw) - \pseudograd(\vw^*)} - 2\inprod{\vw - \vw^*}{\pseudograd(\vw^*)+\delta(\vw)}.
\end{align}
According to the co-coercivity of the pseudo-gradient, we have
\begin{align}
    T_1 
    \leq& \brackets{2\efflr\effLip - 1}\inprod{\vw-\vw^*}{\pseudograd(\vw) - \pseudograd(\vw^*)}- 2\inprod{\vw - \vw^*}{\pseudograd(\vw^*)+\delta(\vw)} \\
    \leq& \brackets{2\efflr\effLip - 1}\inprod{\vw-\vw^*}{\pseudograd(\vw) - \pseudograd(\vw^*)} + \epsilon\vecnorm{\vw - \vw^*} + \frac{1}{\epsilon}\vecnorm{\pseudograd(\vw^*) + \delta(\vw)}
\end{align}
where the last inequality comes from Young's inequality. When $\efflr \effmu \leq \efflr \effLip \leq 1/4$, we have
\begin{align}
    T_1 
    \leq -\frac{\effmu}{2} \vecnorm{\vw - \vw^*}+ \epsilon\vecnorm{\vw - \vw^*} + \frac{1}{\epsilon}\vecnorm{\pseudograd(\vw^*) + \delta(\vw)} 
    \leq \frac{2\vecnorm{\pseudograd(\vw^*) + \delta(\vw)}}{\effmu} \label{eqn:final_t1}
\end{align}
where the last inequality is obtained by setting $\epsilon = \effmu / 2$. Then, substituting \Cref{eqn:final_t1} into \Cref{eqn:mainproof_1} and noting that $\efflr\effmu \leq \efflr\effLip \leq 1/4$ (that is, $\efflr \leq 1/(4\effmu)$), 
\begin{align}
    \E\vecnorm{\vw^+ - \vw^*}
    \leq& (1-\efflr \effmu)\vecnorm{\vw - \vw^*} + \parenth{2\efflr^2 + \frac{2\efflr}{\effmu}}\vecnorm{\pseudograd(\vw^*) + \delta(\vw)} + \efflr^2\Var[\hat{\pseudograd}(\vw)] \\
    \leq& (1-\efflr \effmu)\vecnorm{\vw - \vw^*} + \frac{5\efflr}{2\effmu}\vecnorm{\pseudograd(\vw^*) + \delta(\vw)} + \efflr^2\Var[\hat{\pseudograd}(\vw)] \\
    \leq& (1-\efflr \effmu)\vecnorm{\vw - \vw^*} + \efflr^2\Var[\hat{\pseudograd}(\vw)] + \frac{5\efflr}{\effmu}\vecnorm{\delta(\vw)} + \frac{5\efflr}{\effmu}\vecnorm{\pseudograd(\vw^*)} \\ \label{eqn:mainproof_2}
    \leq& (1-\efflr \effmu)\vecnorm{\vw - \vw^*} + \efflr^2\max_{\vw}\Var[\hat{\pseudograd}(\vw)] \nonumber \\
        &+ \frac{5\efflr}{\effmu}\max_{\vw}\vecnorm{\delta(\vw)} + \frac{5\efflr}{\effmu}\vecnorm{\pseudograd(\vw^*)} 
\end{align}
After total $T$ communication rounds and taking the total expectation, we end up with
\begin{align}
    \E\vecnorm{\vw^{(T)} - \vw^*}
\leq& (1-\efflr \effmu)^T \vecnorm{\vw^{(0)} - \vw^*} + \frac{\efflr}{\effmu} \max_{\vw} \Var[\hat{\pseudograd}(\vw)] \nonumber \\
    &+ \frac{5}{\effmu^2}\max_{\vw} \E_c\vecnorm{\delta_c(\vw)} + \frac{5\rho^2}{\effmu^2}.
\end{align}
When $\eta H \mu \leq 1$, one can easily validate that
\begin{align}
    \effmu 
    = \frac{1-(1-\eta\mu)^H}{\eta H}
    \geq \frac{\mu}{2}.
\end{align}
So it follows that
\begin{align}
    \E\vecnorm{\vw^{(T)} - \vw^*}
\leq& (1-\frac{1}{2}\lr \eta H \mu)^T \vecnorm{\vw^{(0)} - \vw^*} + \frac{2\lr \eta H}{\mu} \max_{\vw} \Var[\hat{\pseudograd}(\vw)] \nonumber \\
    &+ \frac{20}{\mu^2}\max_{\vw} \E_c\vecnorm{\delta_c(\vw)} + \frac{20\rho^2}{\mu^2}.
\end{align}
At last, in order to satisfy $\efflr \effLip \leq 1/4$, one can set $\lr \leq 1/8$, such that
\begin{align}
    \efflr\effLip = \lr \eta H \cdot \frac{1 + (1-\eta \mu)^H}{\eta H} = \lr (1 + (1-\eta \mu)^H) \leq 2\lr \leq \frac{1}{4}.
\end{align}

\section{Bound on Iterate Bias}
In this section, we will provide an upper bound for the iterate bias \Cref{eqn:iterate_bias}. According to the local update rules, we have
\begin{align}
    \norm{\vw_{c,\text{GD}}^{(H)} - \E[\vw_c^{(H)}]}
    =& \norm{\vw_{c,\text{GD}}^{(H-1)} - \E[\vw_c^{(H-1)}] - \eta \nabla F_c(\vw_{c,\text{GD}}^{(H-1)}) + \eta \E[\nabla F_c(\vw_c^{(H-1)})]} \\
    \leq& \norm{\vw_{c,\text{GD}}^{(H-1)} - \E[\vw_c^{(H-1)}] - \eta \nabla F_c(\vw_{c,\text{GD}}^{(H-1)}) + \eta \nabla F_c(\E[\vw_c^{(H-1)}])} \nonumber \\
        &+ \eta \norm{\E[\nabla F_c(\vw_c^{(H-1)})] - \nabla F_c(\E[\vw_c^{(H-1)}])} \\
    \leq& (1-\eta \mu)\norm{\vw_{c,\text{GD}}^{(H-1)} - \E[\vw_c^{(H-1)}]} \nonumber \\
        &+ \eta \norm{\E[\nabla F_c(\vw_c^{(H-1)})] - \nabla F_c(\E[\vw_c^{(H-1)}])}
\end{align}
For the second term, we have
\begin{align}
    \vecnorm{\E[\nabla F_c(\vw_c^{(H-1)})] - \nabla F_c(\E[\vw_c^{(H-1)}])}
    \leq& \E\vecnorm{\nabla F_c(\vw_c^{(H-1)}) - \nabla F_c(\E[\vw_c^{(H-1)}])} \\
    \leq& L^2 \E\vecnorm{\vw_c^{H-1} - \E[\vw_c^{(H-1)}]} \\
    \leq& 2 \eta^ 2 L^2 \sigma^2 (H-1)
\end{align}
where the last inequality comes from previous works~\cite{khaled2020tighter,glasgow2021sharp}. As a result, one can obtain
\begin{align}
    \norm{\vw_{c,\text{GD}}^{(H)} - \E[\vw_c^{(H)}]}
    \leq& (1-\eta \mu)\norm{\vw_{c,\text{GD}}^{(H-1)} - \E[\vw_c^{(H-1)}]} + \sqrt{2}\eta^2 L \sigma (H-1)^{\frac{1}{2}} \\
    \leq& \sqrt{2}\eta^2 L\sigma\sum_{h=0}^{H-1} (1-\eta\mu)^{H-1-h} h^{\frac{1}{2}} \\
    \leq& \sqrt{2}\eta^2 L \sigma \brackets{\sum_{h=0}^{H-1} (1-\eta\mu)^{2(H-1-h)}}^{\frac{1}{2}} \brackets{\sum_{h=0}^{H-1} h}^\frac{1}{2} \\
    \leq& \sqrt{2}\eta^2 L \sigma \brackets{\sum_{h=0}^{H-1} (1-\eta\mu)^{H-1-h}}^{\frac{1}{2}} \brackets{\sum_{h=0}^{H-1} h}^\frac{1}{2} \\
    =& \brackets{\frac{1 - (1-\eta \mu)^H}{\eta \mu H}}^{\frac{1}{2}} \eta^2 L \sigma H (H-1)^{\frac{1}{2}}.
\end{align}
According to the definition of $\delta(\vw)$, we obtain
\begin{align}
    \E_c\vecnorm{\delta_c(\vw)}
    \leq \E_c\vecnorm{\vw_{c,\text{GD}}^{(H)} - \E[\vw_c^{(H)}]}
    \leq& \frac{\effmu \eta^2 L^2 \sigma^2 (H-1)}{\mu} \leq \eta^2 L^2 \sigma^2 (H-1)
\end{align}
where the last inequality comes from the fact that $\effmu \leq \mu$.

\section{Proof of \Cref{corollary:rate}}
\subsection{Deterministic Setting}
When clients perform local GD in each round, there is no stochastic noise. So the error upper bound \Cref{eqn:one_round} can be simplified as follows
\begin{align}
    \vecnorm{\vw^{(T)} - \vw^*}
\leq (1-\frac{1}{2}\lr \eta H \mu)^T \vecnorm{\vw^{(0)} - \vw^*}+ \frac{20\rho^2}{\mu^2}.
\end{align}
If $H\mu \geq L$, then the maximal learning rate is $\eta = 1/H\mu$. When $\lr=1/8$, the upper bound becomes 
\begin{align}
    \vecnorm{\vw^{(T)} - \vw^*}
    \leq& \parenth{1-\frac{1}{16}}^T\vecnorm{\vw^{(0)} - \vw^*} + \frac{20\rho^2}{\mu^2}\\
    \leq& \exp\parenth{-\frac{T}{16}}\vecnorm{\vw^{(0)} - \vw^*}+ \frac{20\rho^2}{\mu^2}.
\end{align}
If $H\mu \leq L$, then the maximal learning rate is $\eta = 1/L$. When $\lr=1/8$, the upper bound becomes
\begin{align}
    \vecnorm{\vw^{(T)} - \vw^*}
    \leq& \parenth{1-\frac{H\mu}{16L}}^T\vecnorm{\vw^{(0)} - \vw^*} + \frac{20\rho^2}{\mu^2}\\
    \leq& \exp\parenth{-\frac{\mu HT}{16L}}\vecnorm{\vw^{(0)} - \vw^*}+ \frac{20\rho^2}{\mu^2}.
\end{align}
We can summarize the above two bounds as follows:
\begin{align}
    \vecnorm{\vw^{(T)} - \vw^*}
    \leq& \exp\parenth{-\frac{T}{16\kappa}\min\{\kappa, H\}}\vecnorm{\vw^{(0)} - \vw^*}+ \frac{20\rho^2}{\mu^2} \\
    =& \gO\parenth{\exp\parenth{-\frac{T}{16\kappa}\min\{\kappa, H\}} + \rho^2}.
\end{align}

\subsection{Stochastic Setting}
Substituting the upper bounds for $\Var[\pseudograd(\vw)]$ and $\delta(\vw)$ into \Cref{eqn:mainproof_2} and setting $\lr = 1/8$,
\begin{align}
    \E\vecnorm{\vw^{(t+1)} - \vw^*}
    \leq& (1-\efflr \effmu)\vecnorm{\vw^{(t)} - \vw^*} + \efflr^2\frac{2\sigma^2}{MH} + \frac{5\efflr^3}{\effmu}\frac{\sigma^2 L^2(H-1)}{\lr^2 H^2} + \frac{5\efflr}{\effmu}\vecnorm{\pseudograd(\vw^*)} \\
    \leq& (1-\efflr \effmu)\vecnorm{\vw^{(t)} - \vw^*} + \efflr^2\frac{2\sigma^2}{MH} + \efflr^3\frac{320 \sigma^2 L^2 }{\effmu H} + \efflr \frac{5\rho^2}{\effmu}.
\end{align}
After minor rearrangement, we can get
\begin{align}
    \E\vecnorm{\vw^{(t+1)} - \vw^*} - \frac{5\rho^2}{\effmu^2}
    \leq (1-\efflr \effmu)\brackets{\vecnorm{\vw^{(t)} - \vw^*} - \frac{5\rho^2}{\effmu^2}} + \efflr^2 \frac{2\sigma^2}{MH} + \efflr^3 \frac{320 \sigma^2 L^2}{\effmu H}.
\end{align}
After total $T$ communication rounds,
\begin{align}
    \E\vecnorm{\vw^{(t+1)} - \vw^*} - \frac{5\rho^2}{\effmu^2}
    \leq (1-\efflr \effmu)^T\underbrace{\brackets{\vecnorm{\vw^{(0)} - \vw^*} - \frac{5\rho^2}{\effmu^2}}}_{r_0} +  \frac{2\efflr\sigma^2}{\effmu MH} +  \frac{320 \efflr^2\sigma^2 L^2}{\effmu^2 H}.
\end{align}
If we set $\efflr = \frac{\nu}{\mu T}$, where $\nu = 2\ln(\max\{r_0 \mu^2 M H T / (8\sigma^2), r_0 \mu^4 H T^2 / (1280 L^2\sigma^2)\})$, then it follows that
\begin{align}
    \E\vecnorm{\vw^{(t+1)} - \vw^*} - \frac{20\rho^2}{\mu^2} 
    \leq& \frac{8 \sigma^2\nu}{\mu^2 M H T} + \frac{1280 \sigma^2 L^2 \nu}{\mu^4 H T^2} + \exp\parenth{ - \frac{\nu}{2}}r_0 \\
    \leq& \frac{8 \sigma^2(\nu+1)}{\mu^2 M H T} + \frac{1280 \sigma^2 L^2 (\nu+1)}{\mu^4 H T^2} \\
    =& \widetilde{\gO}\parenth{\frac{\sigma^2}{M H T} + \frac{\sigma^2}{H T^2}}
\end{align}
where $\widetilde{\gO}$ omits the logarithmic factors.

\section{Proof of \Cref{corollary:rate_quadratic}}
In the quadratic setting, the iterate bias is zero. So from \Cref{eqn:mainproof_2}, we get
\begin{align}
    \E\vecnorm{\vw^{(t+1)} - \vw^*} - \frac{5\rho^2}{\effmu^2}
    \leq& (1-\efflr \effmu)\brackets{\vecnorm{\vw^{(t)} - \vw^*} - \frac{5\rho^2}{\effmu^2}} + \efflr^2 \frac{2\sigma^2}{MH}.
\end{align}
Then, taking the total expectation on both sides, for the $T$-th round, we have
\begin{align}
    \E\vecnorm{\vw^{(T)} - \vw^*} - \frac{5\rho^2}{\effmu^2}
    \leq& (1-\efflr \effmu)^T \brackets{\vecnorm{\vw^{(0)} - \vw^*} - \frac{5\rho^2}{\effmu^2}} +  \frac{2\efflr\sigma^2}{\effmu MH}
\end{align}
According to \cite{stich2019unified}, with some constant learning rate $\efflr = \gO(1/\effmu T)$, one can directly obtain
\begin{align}
    \E\vecnorm{\vw^{(t+1)} - \vw^*} - \frac{5\rho^2}{\effmu^2}
    = \widetilde{\gO}\parenth{\exp\parenth{-\effmu \efflr_{\text{max}} T} + \frac{2\sigma^2}{\effmu^2 MH T}}.
\end{align}
where $\widetilde{\gO}$ omits the logarithmic factors. In addition, note that $\eta \leq \min\{1/(\mu H), 1/L\}$. Accordingly, $\efflr_{\max} = \min\{1/8,\mu H/(8L)\}$ and $\effmu > \mu /2$. As a consequence,
\begin{align}
    \E\vecnorm{\vw^{(t+1)} - \vw^*}
    =& \widetilde{\gO}\parenth{\exp\parenth{-\frac{T}{16\kappa}\min\{\kappa, H\}} + \frac{8\sigma^2}{\mu^2 MH T} + \frac{20\rho^2}{\mu^2}} \\
    =& \widetilde{\gO}\parenth{\exp\parenth{-\frac{T}{16\kappa}\min\{\kappa, H\}} + \frac{\sigma^2}{MH T} + \rho^2}.
\end{align}

\end{document}